\pdfoutput=1  
\PassOptionsToPackage{table}{xcolor}

\documentclass{bytedance_seed}

\usepackage[scaled]{helvet}
\DeclareFontShape{T1}{bytesans}{m}{n}{<->ssub * phv/m/n}{}
\DeclareFontShape{T1}{bytesansmedium}{m}{n}{<->ssub * phv/m/n}{}
\renewcommand\affiliationformat[2][]{{\affiliationfont{\sffamily $^{#1}$#2}}}

\usepackage{amsmath,amssymb,amsthm}
\usepackage{tabularx}
\usepackage{algorithm}
\usepackage{algorithmic}
\usepackage{float}
\usepackage{xspace}
\usepackage{pifont}

\AtBeginDocument{\setlength{\parskip}{3pt plus 1pt}}
\newlength{\bodyfigwidth}

\newlength{\headerlogoheight}
\newlength{\headerlogoraise}

\fancypagestyle{firststyle}{%
  \fancyhead[R]{}%
  \fancyhead[L]{%
    \vskip -\headerlogoraise
    \includegraphics[height=\headerlogoheight]{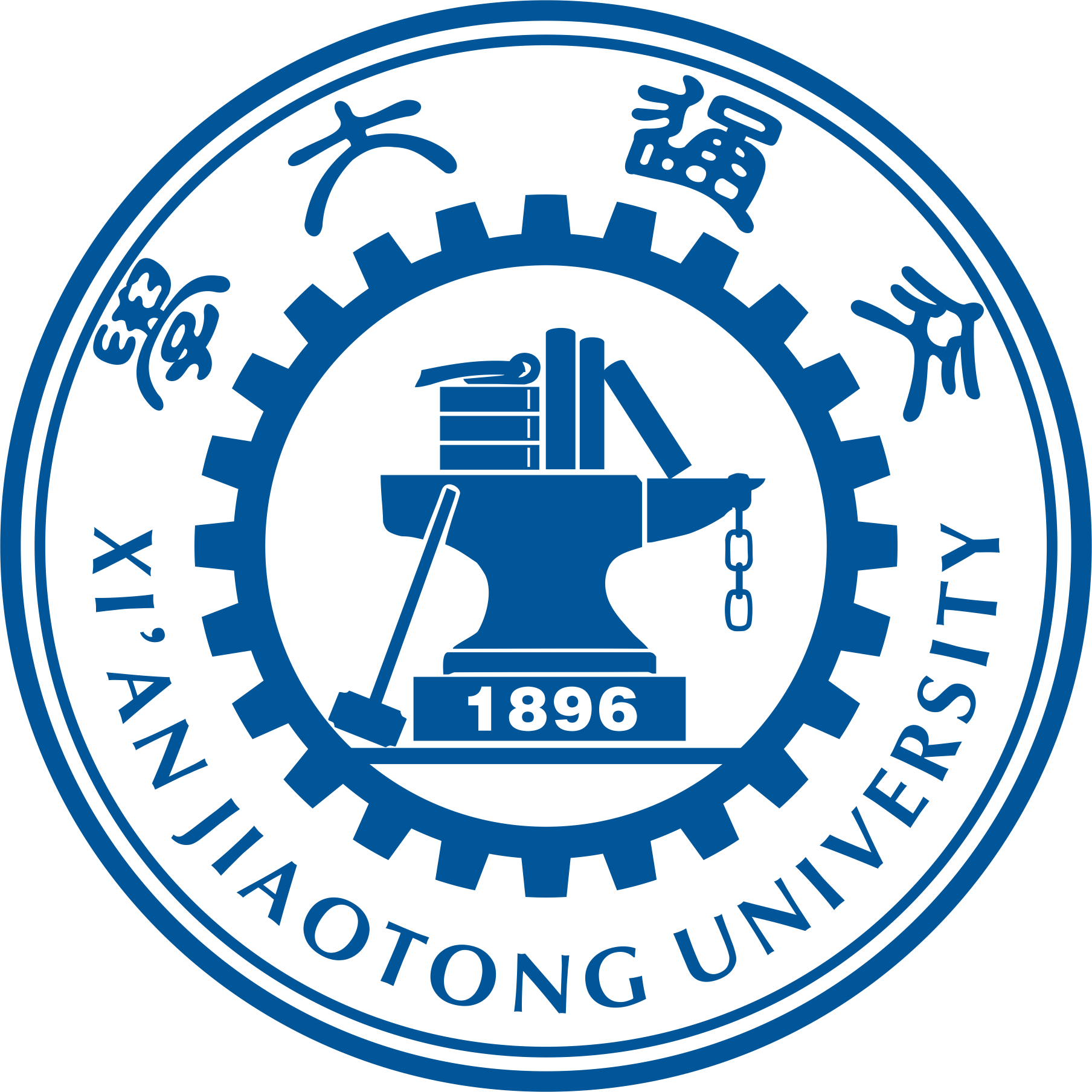}%
    \hspace{6mm}%
    \includegraphics[height=\headerlogoheight]{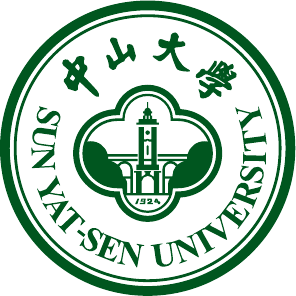}%
  }%
}

\newtheorem{theorem}{Theorem}

\theoremstyle{definition}

\theoremstyle{remark}

\theoremstyle{plain}

\definecolor{bestrow}{HTML}{E8F4FD}
\definecolor{deltarow}{HTML}{F5F5F8}

\newcommand{\KL}{D_{\mathrm{KL}}}

\newcommand{\method}{Visual-OPSD\xspace}
\newcommand{\methodfull}{Visual On-Policy Self-Distillation\xspace}
\newcommand{\methodnoise}{Visual-OPSD-Noise\xspace}

\title{\raisebox{-0.22\height}{\includegraphics[height=2.0em]{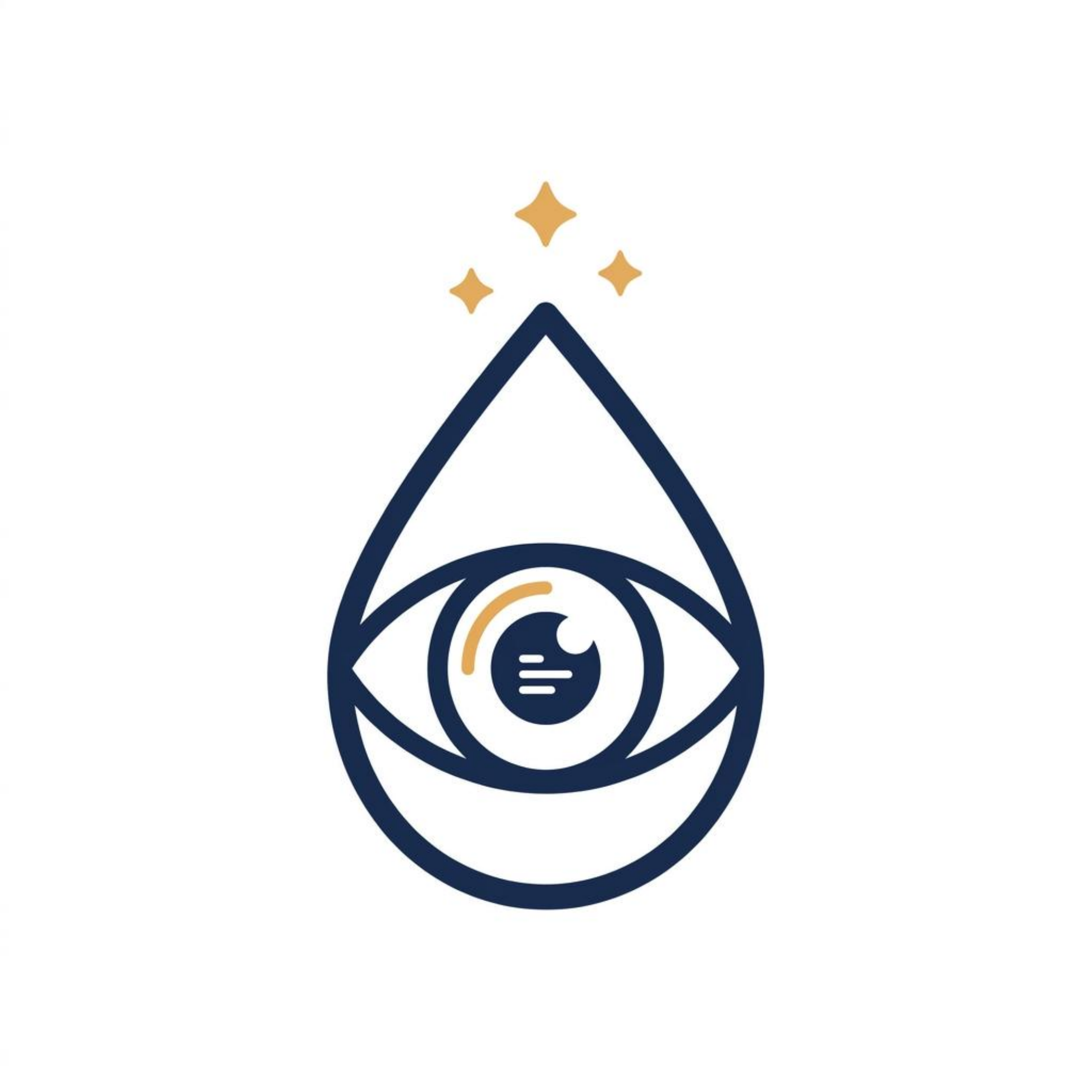}}\hspace{0.2em}Visual-OPSD: Cross-Modal On-Policy Self-Distillation \\ for Efficient Unified Multimodal Reasoning}

\author[1,3,\dagger]{Pengyu Li}
\author[1,3]{Zhitao Gao}
\author[1,2,*]{Lingling Zhang}
\author[1,3]{Muye Huang}
\author[4]{Yuanming Li}
\author[1,3]{Zesheng Yang}
\author[1,3]{Fangzhi Xu}
\author[1,2]{Jun Liu}

\affiliation[1]{Xi'an Jiaotong University}
\affiliation[2]{MOE KLINNS Lab, Xi'an Jiaotong University}
\affiliation[3]{Shaanxi Province Key Laboratory of Big Data Knowledge Engineering}
\affiliation[4]{Sun Yat-sen University}

\contribution[*]{Corresponding author}

\abstract{
Unified multimodal models (UMMs) interleave generated ``visual thoughts'' (VTs) with text reasoning to improve spatial tasks, at roughly an order-of-magnitude inference cost from multi-step diffusion. We find that this cost buys little at inference: on ThinkMorph, removing or noising the intermediate VTs leaves accuracy essentially unchanged across nine benchmarks, and once a VT is rendered the subsequent reasoning attends to it regardless of what it contains. A KL diagnostic nonetheless shows that conditioning on a privileged VT trace shifts the model's completion distribution, which places the useful reasoning in the generation pathway rather than in the rendered pixels. We exploit that gap with \methodfull (\method). Teacher and student share identical weights and differ only in context: the teacher attends to privileged VTs, the student only to the problem image and question. Token-level JSD distillation along on-policy student trajectories moves the teacher's reasoning into a text-only student. Across nine benchmarks, \method improves on its generative teacher by $+3.40$pp while running $14.3\times$ faster (10.0s vs.\ 142.8s per sample), and exceeds same-scale VLMs by $+63.83$pp on VSP. Replacing real VTs with Gaussian noise recovers only $+0.40$pp of the $+10.28$pp gain, and distillation closes $58.4\%$ of the measured KL gap on held-out samples.
}

\correspondence{Pengyu Li$^{\dagger}$ at \email{lipengyu.tiez@stu.xjtu.edu.cn}}
\checkdata[Code]{\url{https://github.com/TiezMind/Visual-OPSD}}

\begin{document}

\maketitle


\begin{figure}[t]
    \centering
    \includegraphics[width=\bodyfigwidth]{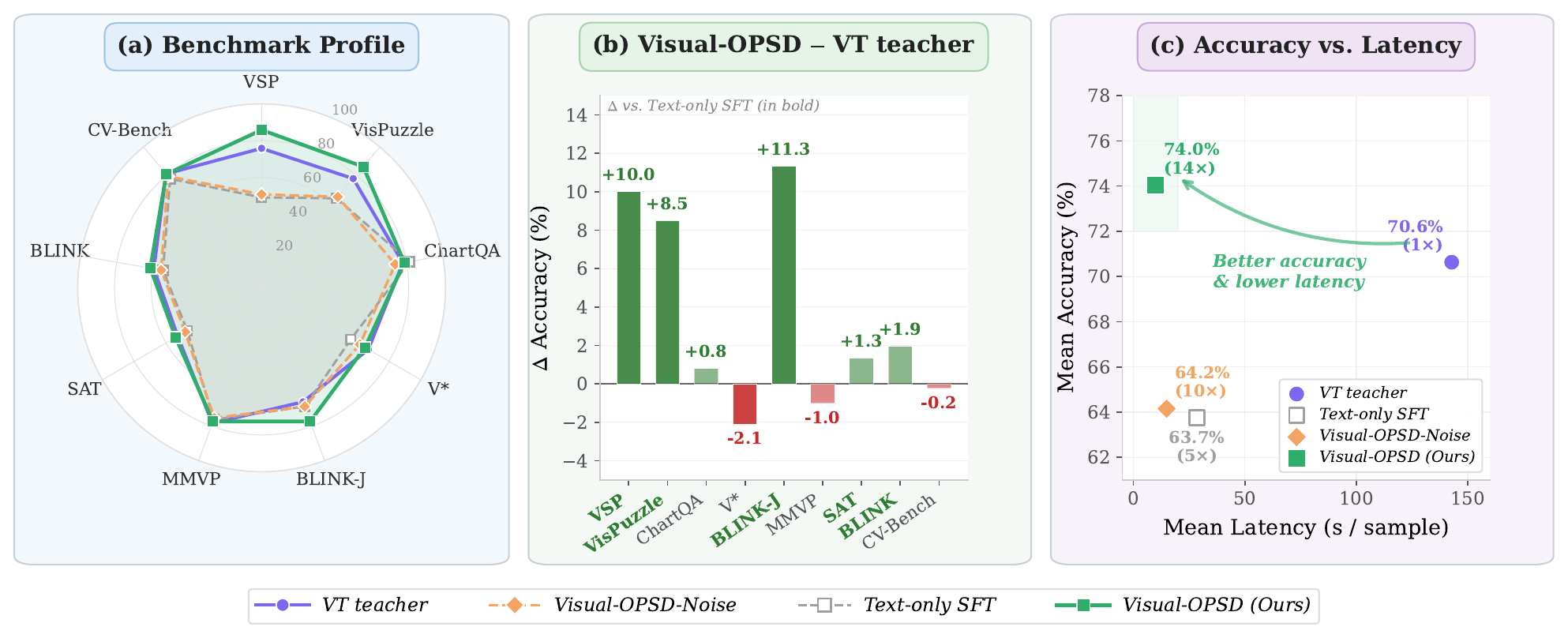}
    \caption{%
        \textbf{\method matches its VT-generating teacher at $14\times$ lower latency.}
        \textbf{(a)}~Radar over 9 tasks: \method (green) $\geq$ teacher (purple) on 6/9.
        \textbf{(b)}~Largest per-task gains: VSP $+10.0$, VisPuzzle $+8.5$, BLINK-J $+11.3$.
        \textbf{(c)}~Accuracy--latency Pareto: \method 74.0\%/10.0s vs.\ teacher 142.8s.
    }
    \label{fig:main_results}
\end{figure}

\begin{figure*}[t]
    \centering
    \includegraphics[width=\bodyfigwidth]{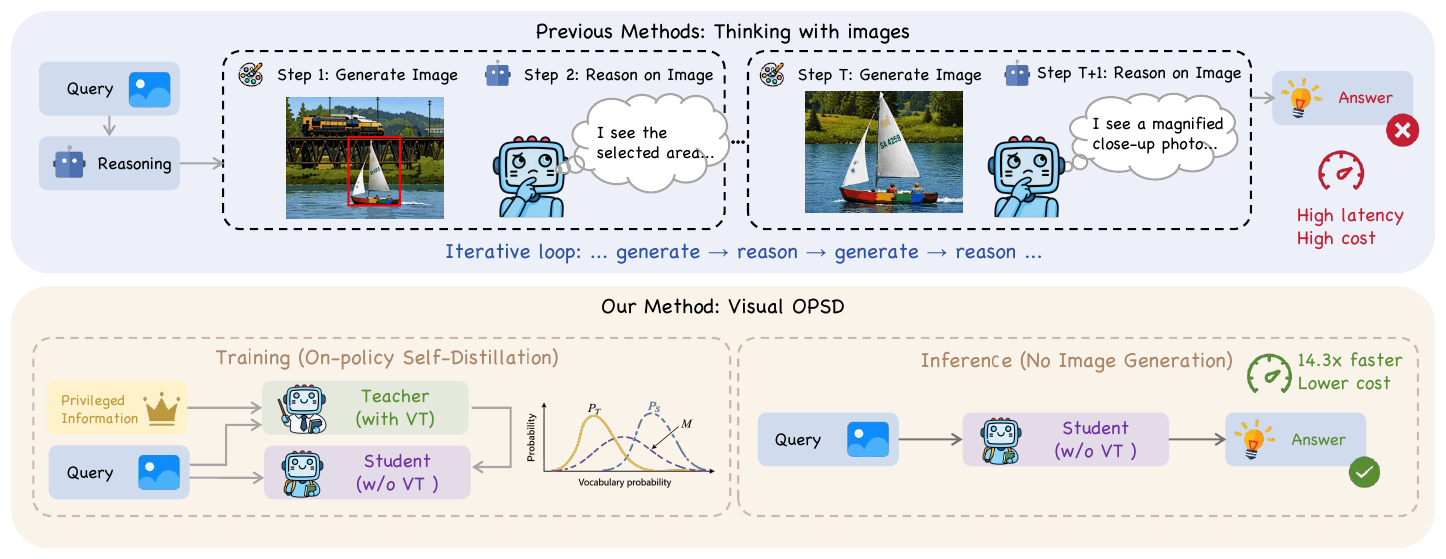}
    \caption{%
        \textbf{Prior interleaved visual CoT vs.\ \method.}
        \textbf{(Top)}~Previous methods iterate \emph{generate-then-reason}, rendering each visual thought (VT) via 50-step diffusion at high latency and cost.
        \textbf{(Bottom)}~\method distills the generation pathway into a text-only student via cross-modal on-policy self-distillation, yielding $+3.40$pp accuracy at $14.3\times$ speedup with no image generation at inference.
    }
    \label{fig:compare}
\end{figure*}

\section{Introduction}
\label{sec:intro}

Unified multimodal models (UMMs)~\citep{deng2025emerging,gu2026thinkmorph,team2024chameleon,wang2024emu3} handle visual understanding and generation within a single set of weights, and they turn out to support interleaved visual chain-of-thought reasoning: the model alternates text segments with intermediate ``visual thoughts'' (VTs) generated by multi-step diffusion before committing to an answer. ThinkMorph~\citep{gu2026thinkmorph} reports consistent spatial-reasoning gains from this protocol over text-only reasoning, which it attributes to richer intermediate representations.

The protocol is expensive. Each VT costs 50 diffusion denoising steps, putting per-sample latency roughly an order of magnitude above text-only reasoning. The cost raises a question prior to efficiency: do the rendered VT pixels carry load-bearing information at all, or do the trajectory-level gains come from somewhere else?

We investigate this with a controlled pilot study on ThinkMorph (Figure~\ref{fig:pilot}) that intervenes on intermediate VTs at inference time without retraining. Removing the VTs entirely and letting the model reason in text alone preserves most of the accuracy across nine benchmarks ($69.90$ vs.\ $70.63$ on average), and on BLINK-J and MMVP text-only reasoning is slightly ahead of the full interleaved setting; replacing real VTs with Gaussian noise behaves much the same. A per-layer attention analysis on V* (Figure~\ref{fig:attention-layer}) adds that the text reasoning following a VT attends almost exclusively to that VT and largely ignores the original input, whatever the VT depicts. The rendered pixels therefore contribute little beyond what text-only reasoning already recovers, and the diffusion cost is not commensurate with their direct benefit. What generation training does buy shows up elsewhere: even in text-only mode ThinkMorph averages $69.90$, well above same-scale dedicated VLMs that were never trained with a generation pathway (InternVL3.5-8B $58.04$, Qwen3-VL-8B $64.58$). The value of that pathway thus sits somewhere other than in the pixels it renders.

Our hypothesis is that it sits in the internal representations that generation training shaped. To test it we measure the KL divergence between the model's completion distributions with and without a privileged VT trace. The divergence is large in every task category (Section~\ref{sec:kl-diag}), so a distributional gap is available for distillation to exploit even though the VTs themselves are not load-bearing at inference.

Building on these observations, we propose \methodfull (\method), a cross-modal on-policy self-distillation framework that exploits the above distributional gap within a single set of weights. Teacher and student share identical parameters and differ only in conditioning context: the teacher attends to privileged VT images, the student only to the problem image and question. Token-level JSD distillation along on-policy student trajectories transfers the generation pathway's distributional knowledge into the student. At inference the student runs in text-only mode, with no diffusion steps, architectural changes, or additional parameters.

Switching ThinkMorph to text-only inference already removes the diffusion latency at no cost, so \method has to justify itself against that option and not only against the interleaved teacher. It does: the distilled student averages $74.03$ against $69.90$ for zero-shot text-only inference on the same weights, a gain of $+4.13$pp.

Our contributions are as follows:

\begin{itemize}
    \item \textbf{Finding.} Controlled interventions on ThinkMorph show that rendered VT pixels are not load-bearing at inference, while the generation pathway still carries a large distributional signal that a KL diagnostic can measure. Visual generation training and inference utility come apart in UMMs, which to our knowledge has not been examined.
    \item \textbf{Method.} We propose \method, a cross-modal on-policy self-distillation framework that transfers this distributional knowledge from a VT-conditioned teacher to a text-only student within a single model. To our knowledge, this is the first OPSD instance bridging the asymmetry between generation and understanding in a unified multimodal model.
    \item \textbf{Results.} Across nine benchmarks, \method preserves or improves accuracy on 6 of 9 tasks ($+3.40$pp on average) while cutting per-sample inference time by $14.3\times$. A noise control gains only $+0.40$pp over text-only fine-tuning, against $+10.28$pp for \method, which places the source of the gain in the semantic content of the generation pathway rather than in regularization. We release the training code, evaluation scripts, and distilled checkpoints.
\end{itemize}

\begin{figure*}[t]
    \centering
    \begin{subfigure}[t]{0.51\textwidth}
        \centering
        \includegraphics[width=\textwidth]{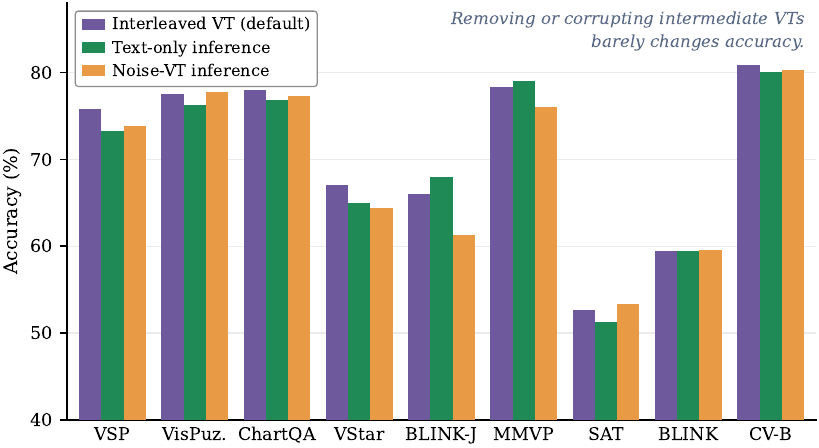}
        \caption{Inference-time VT intervention on ThinkMorph.}
        \label{fig:pilot}
    \end{subfigure}
    \hfill
    \begin{subfigure}[t]{0.46\textwidth}
        \centering
        \includegraphics[width=\textwidth]{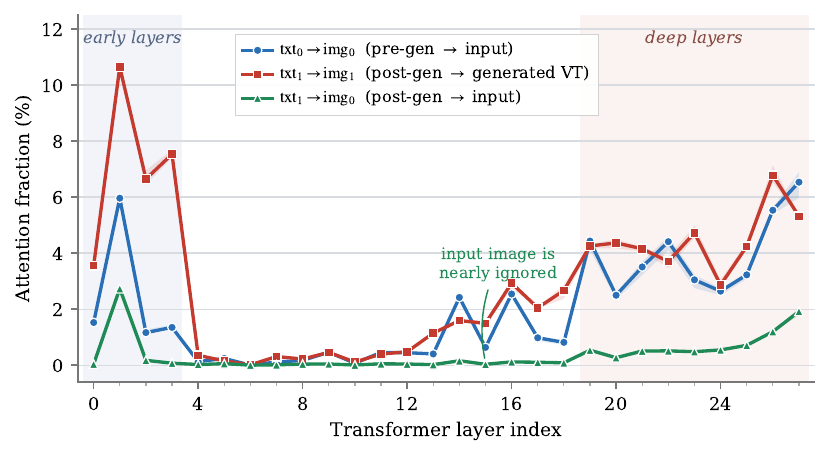}
        \caption{Per-layer cross-modal attention on V*.}
        \label{fig:attention-layer}
    \end{subfigure}
    \caption{%
        \textbf{Two diagnostics on ThinkMorph that motivate \method.}
        \textbf{(a)}~Removing or corrupting intermediate VTs at inference leaves accuracy largely unchanged across all nine benchmarks.
        \textbf{(b)}~Once generated, a VT dominates the subsequent reasoning attention regardless of its content.
    }
    \label{fig:pilot-attention}
\end{figure*}

\section{Method}
\label{sec:method}

\subsection{Preliminaries}
\label{sec:prelim}

\paragraph{Unified multimodal model.} \method applies to any UMM that supports both visual understanding and generation. We instantiate it on ThinkMorph~\citep{gu2026thinkmorph}, built on the BAGEL architecture~\citep{deng2025emerging}, which fuses a Qwen2.5 LLM backbone with MoT decoder layers, a SigLIP-so400m NaViT vision encoder, and a FLUX VAE for latent image encoding and generation, so that one set of weights covers both image$\to$text understanding and text$\to$image generation.

\paragraph{Interleaved chain-of-thought protocol.} The model performs reasoning via an interleaved protocol in which text and generated images alternate. Writing $x$ for the problem input (image and question), the thought sequence is $\mathcal{T} = (\hat{m}_1, \ldots, \hat{m}_n)$ with $\hat{m}_i \sim \mathcal{P}_\theta(\cdot \mid x, \hat{m}_{<i})$, where each $\hat{m}_i \in \{\hat{t}_i, \hat{v}_i\}$ is either a text thought or a visual thought. Special tokens are omitted from this notation, but modality transitions are controlled in practice by delimiters: image thoughts are bracketed by \texttt{<image\_start>} and \texttt{<image\_end>}, which lets the model switch between textual and visual reasoning within a single sequence. Section~\ref{sec:vpd} uses the same pair to suppress visual thoughts during sampling.

\paragraph{Generation cost.} Each VT costs 50 denoising steps through the diffusion pathway, roughly a $14\times$ latency overhead: 142.8s per sample with VT generation against 10.0s without. The goal is to keep the knowledge that generation training produced while dropping the generation step at inference.

\subsection{Measuring Generation Knowledge: KL Diagnostic}
\label{sec:kl-diag}

We first formalize and measure the ``generation knowledge'' hypothesis. Consider the frozen unified model $M$ with parameters $\theta$. We construct two forward passes that share \emph{identical completion tokens} $\mathbf{c} = (c_1, \ldots, c_T)$ but differ in their conditioning context:
\begin{align}
    \mathcal{C}_T &= [\text{sys},\; \text{ViT}(\text{img}),\; \text{question},\; \text{ref\_intro},\; (\text{ViT}(\text{VT}_i))^{+},\; \text{transition}] \label{eq:ctx-teacher}\\
    \mathcal{C}_S &= [\text{sys},\; \text{ViT}(\text{img}),\; \text{question}] \label{eq:ctx-student}
\end{align}
The teacher context $\mathcal{C}_T$ prepends a \emph{strictly visual-only} privileged reasoning trace before the completion. Here, $\text{ref\_intro}$ is a preamble that frames the subsequent images as privileged visual references, and $\text{transition}$ instructs the model to now reason independently from this privileged context (full prompts in Appendix~\ref{app:prompts}). The student context $\mathcal{C}_S$ contains only the problem image and question. Both paths autoregressively process the \emph{same} completion tokens $\mathbf{c}$ yet produce different next-token distributions at each position, $p_\theta(c_t \mid c_{<t}, \mathcal{C}_T) \neq p_\theta(c_t \mid c_{<t}, \mathcal{C}_S)$. We define the generation knowledge for a sample as the average per-token KL divergence over the shared completion span:
\begin{equation}
\label{eq:gen-knowledge}
    \mathcal{K}_{\text{gen}} \triangleq \frac{1}{T}\sum_{t=1}^{T} \KL\!\Big(p_\theta(\cdot \mid c_{<t}, \mathcal{C}_T) \;\Big\|\; p_\theta(\cdot \mid c_{<t}, \mathcal{C}_S)\Big)
\end{equation}
This quantity measures how much the VT reasoning trace shifts the model's next-token predictions on a fixed completion, and serves as a proxy for the size of the distributional gap that \method training seeks to close. The noise-control variant (Section~\ref{sec:vt-quality}) and the post-distillation gap-closing analysis (Appendix~\ref{app:kl-closing}) further attribute this gap to the semantic content of the VT reasoning trace.

We evaluate the diagnostic on a held-out set of 1,000 examples (250 per category) spanning the four task categories, disjoint from the data used for training:

\begin{table}[h]
\caption{KL diagnostic on shared completion tokens, measured on 1,000 held-out examples. Large $\mathcal{K}_{\text{gen}}$ indicates substantial VT-encoded knowledge in the completion distribution, which is the learning signal available to \method.}
\label{tab:kl-diag}
\begin{center}
\small
\begin{tabular}{lcc}
\toprule
\textbf{Task Category} & \textbf{Samples} & $\boldsymbol{\mathcal{K}_{\text{gen}}}$ \textbf{(nats/token)} \\
\midrule
Visual Search & 250 & 4.23 \\
Spatial Navigation & 250 & 3.96 \\
Chart Refocus & 250 & 3.51 \\
Jigsaw Assembly & 250 & 6.84 \\
\midrule
\textbf{Overall} & \textbf{1,000} & \textbf{4.64} \\
\bottomrule
\end{tabular}
\end{center}
\end{table}

All categories exhibit $\mathcal{K}_{\text{gen}} \gg 0$ (overall $4.64$ nats/token), indicating that the VT reasoning trace systematically shifts the model's completion predictions. Jigsaw Assembly shows the largest gap ($6.84$ nats/token), consistent with spatial-manipulation tasks deriving the most benefit from intermediate visual reasoning.

A per-token analysis (Figure~\ref{fig:kl-heatmap} in Appendix~\ref{app:kl-heatmap}) shows the divergence to be \emph{non-uniform}: it concentrates on tokens encoding spatial relations, quantities, and visually grounded answers, while function words carry near-zero divergence. VT reasoning therefore informs the predictions that decide task success rather than shifting overall style. Appendix~\ref{app:kl-closing} shows that \method training closes $58.4\%$ of this gap on the same held-out set.

\subsection{\method: Cross-Modal On-Policy Self-Distillation}
\label{sec:vpd}

\paragraph{Cross-modal information gap.} \method distills generation knowledge by exploiting the cross-modal information gap between teacher and student contexts within the same model. The two sequences are $\mathcal{C}_T$ and $\mathcal{C}_S$ from Eqs.~\ref{eq:ctx-teacher}--\ref{eq:ctx-student}, each followed by the same completion span, and loss is active on that span alone. The privileged channel is \emph{strictly visual-only}: the teacher's extra context holds the intermediate VT images (encoded via ViT) and nothing else. Since the completion tokens are identical on both sides, the teacher's logits differ only because its KV cache encodes those images, and that difference is the generation pathway's distillable knowledge. Figure~\ref{fig:pipeline} shows the training loop; Algorithm~\ref{alg:opsd-vpd} in Appendix~\ref{app:impl} gives the pseudocode.

\begin{figure*}[t]
    \centering
    \includegraphics[width=\bodyfigwidth]{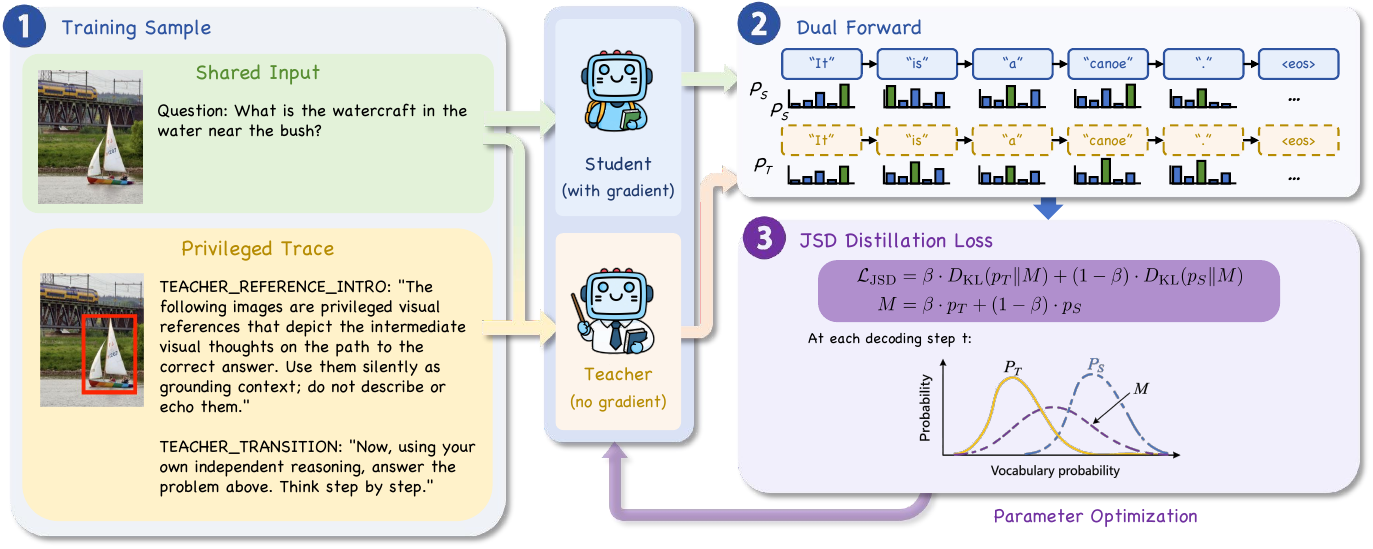}
    \caption{%
        \textbf{Overview of \method.} From the same UMM, a \textcolor{green!50!black}{student} $\pi_\theta(\cdot\!\mid\!\mathcal{C}_S)$ (gradients on) sees only $[\text{sys},\text{ViT}(x),q]$, while an EMA \textcolor{orange!80!black}{teacher} $\pi_{\bar\theta}(\cdot\!\mid\!\mathcal{C}_T)$ (no gradient) additionally receives privileged visual thoughts $(\text{ViT}(\hat v_i))^+$. The student samples $\hat c\!\sim\!\pi_\theta$ on-policy; both policies rescore the shared completion to yield $p_S^{(t)},p_T^{(t)}$, optimized by per-token JSD. At inference, the student runs text-only with no VT generation, $14.3\!\times$ faster, and $+3.40$pp over the generative teacher.
    }
    \label{fig:pipeline}
\end{figure*}

\paragraph{On-policy sampling.} At each step the student generates a completion from its current policy rather than reading ground-truth text. When the student emits \texttt{<image\_start>}, which would enter generation mode, we immediately append the closing \texttt{<image\_end>} to the running context and resume text sampling, so the image span is emptied and the diffusion pathway is never invoked; this keeps sampling on-policy while preventing collapse into generation mode. The injected \texttt{<image\_end>} steers sampling only and never enters the completion $\mathbf{c}$ that the two policies later rescore. No divergence term is therefore computed at that position, and the teacher is never asked to score a position where its context leads it to expect latent visual tokens.

\paragraph{Training objective.} Given an on-policy completion $\mathbf{c} = (c_1, \ldots, c_T)$ sampled from the student, both teacher and student evaluate this shared sequence and produce next-token distributions $p_T^{(t)} \triangleq p_{\bar\theta}(\cdot \mid c_{<t}, \mathcal{C}_T)$ and $p_S^{(t)} \triangleq p_\theta(\cdot \mid c_{<t}, \mathcal{C}_S)$ at each position $t$. We minimize the trajectory-averaged, token-wise generalized Jensen--Shannon divergence:
\begin{equation}
\label{eq:vpd-loss}
\begin{aligned}
    \mathcal{L}(\theta) &= \frac{1}{T}\sum_{t=1}^{T} \text{JSD}_\beta\!\left(p_T^{(t)} \,\big\|\, p_S^{(t)}\right), \\
    \text{JSD}_\beta(p_T \| p_S) &= \beta \KL(p_T \| m) + (1\!-\!\beta) \KL(p_S \| m), \quad m = \beta p_T + (1\!-\!\beta) p_S,
\end{aligned}
\end{equation}
where $\beta = 0.5$ gives symmetric JSD (our default) and gradients flow only through $p_S$. Standard distillation interpolates soft targets with hard CE labels; here the teacher--student gap is one of \emph{information} rather than capacity, so we use pure JSD and no CE term that would anchor the student to fixed ground-truth completions.

\paragraph{Top-$K$ support.} For memory we restrict the divergence at each position to the teacher's top-$K$ token indices $\mathcal{V}_K^{(t)}$ ($K = 256$), cutting the vocabulary dimension from 152K to 256 with negligible effect on distillation quality (Appendix~\ref{app:jsd-ablation}). Both arguments are renormalized over that support, $\tilde p^{(t)}(v) = p^{(t)}(v) / \sum_{u \in \mathcal{V}_K^{(t)}} p^{(t)}(u)$, so each is a proper distribution on $\mathcal{V}_K^{(t)}$ and the divergence is well defined. A position whose sampled token $c_t$ falls outside $\mathcal{V}_K^{(t)}$ is not skipped: the objective matches distributions rather than scoring the realized token, so such a position contributes like any other, with the off-support mass of $p_S^{(t)}$ removed by renormalization.

\paragraph{Pointwise clipping.} Raw per-position divergence is heavily skewed, and a few extreme disagreements can dominate the gradient. We clip \emph{pointwise in the vocabulary} rather than per token, giving the implemented objective
\begin{equation}
\label{eq:vpd-loss-impl}
\begin{aligned}
    \mathcal{L}(\theta) &= \frac{1}{T}\sum_{t=1}^{T} \sum_{v \in \mathcal{V}_K^{(t)}} \min\!\big(d_t(v),\; \tau\big), \qquad \tau = 0.05, \\
    d_t(v) &= \beta\, \tilde p_T^{(t)}(v) \log\frac{\tilde p_T^{(t)}(v)}{m_t(v)} + (1\!-\!\beta)\, \tilde p_S^{(t)}(v) \log\frac{\tilde p_S^{(t)}(v)}{m_t(v)} \;\geq\; 0,
\end{aligned}
\end{equation}
with $m_t = \beta \tilde p_T^{(t)} + (1\!-\!\beta) \tilde p_S^{(t)}$. At $\beta = 0.5$ a single entry contributes at most $\tfrac{1}{2}\log 2 \approx 0.347$ nats, so $\tau$ sits at $14\%$ of the per-entry range and the clamp binds only where the two policies disagree sharply on one alternative. A position whose divergence is spread across many entries, which Figure~\ref{fig:kl-heatmap} associates with content tokens, keeps its full token-level signal and its total may exceed $\tau$ by orders of magnitude, because the clamp applies entry by entry. Eq.~\ref{eq:vpd-loss-impl} is a clipped surrogate rather than a divergence; it agrees with Eq.~\ref{eq:vpd-loss} exactly on positions where every entry falls below $\tau$. Removing the clamp costs $3.2$pp of average accuracy on the three spatial benchmarks of Appendix~\ref{app:jsd-ablation}.


\subsection{A Distributional View of Teacher--Student Transfer}
\label{sec:theory}

\method exceeds its generative teacher by $+3.40$pp while using less information at inference (Table~\ref{tab:main}). A VT-conditioned prediction carries both useful reasoning and whatever errors the generated image introduced, and Figure~\ref{fig:attention-layer} shows that later reasoning concentrates on that image, which can amplify the errors. Ensemble distillation gives a distributional account of transferring teacher information without conditioning inference on any single generated VT.

Fix the input $(X,Q)$, completion prefix $c_{<t}$, and teacher weights $\bar\theta$. For $V \sim p_{\text{gen}}(\cdot \mid X,Q)$, write $p_V(v)=p_{\bar\theta}(v\mid c_{<t},\mathcal{C}_T(V))$ for the teacher's probability of vocabulary entry $v$, and define the arithmetic-mean teacher $\bar p(v)=\mathbb{E}_V[p_V(v)]$. Consider strictly positive teacher distributions on a finite vocabulary and a student distribution $q$ that is freely optimized on the probability simplex.

\begin{theorem}[Ensemble distillation targets]
\label{thm:signal-noise}
For every strictly positive student distribution $q$,
\begin{equation}
\label{eq:kl-ensemble}
\mathbb{E}_V[\KL(p_V\|q)]
=\mathbb{E}_V[\KL(p_V\|\bar p)]+\KL(\bar p\|q).
\end{equation}
Thus the expected forward-KL objective has the unique minimizer $q^\star=\bar p$. For symmetric JSD,
\begin{equation}
\label{eq:jsd-noise}
\mathbb{E}_V[\mathrm{JSD}_{1/2}(p_V\|q)]
\geq \mathrm{JSD}_{1/2}(\bar p\|q).
\end{equation}
The JSD inequality gives a lower bound, not an identification of its minimizer. The proof is given in Appendix~\ref{app:proofs}.
\end{theorem}

The analysis concerns full distributions and a fixed teacher, whereas \method uses top-$K$ truncation, clipping, and EMA updates. It supports reading distillation as the transfer of aggregate teacher information rather than the reproduction of any individual VT-conditioned prediction. The contrast between \methodnoise ($+0.40$pp) and \method ($+10.28$pp) points to semantic VT content as what makes that transfer work, and not depending on generated pixels at inference is one possible explanation for the student's advantage.

\section{Experiments}
\label{sec:exp}

\subsection{Experimental Setup}

\paragraph{Benchmarks.} We evaluate on 9 benchmarks in two regimes. The \textbf{VT-useful} group covers spatial reasoning, where generation knowledge should matter most: VSP~\citep{wu2024vsp} ($n$=120), VisPuzzle~\citep{gu2026thinkmorph} ($n$=200), BLINK-Jigsaw~\citep{fu2024blink} ($n$=150), SAT~\citep{ray2024sat} ($n$=150), and BLINK~\citep{fu2024blink} ($n$=671). The \textbf{VT-neutral} group covers general VLM tasks: ChartQA~\citep{masry2022chartqa} ($n$=944), V*~\citep{wu2024vstar} ($n$=191), MMVP~\citep{tong2024mmvp} ($n$=150), and CV-Bench~\citep{tong2024cvbench} ($n$=700).

\paragraph{Compared.} \textbf{BAGEL-7B} is the base UMM with no task-specific fine-tuning. \textbf{Text-only SFT} fine-tunes BAGEL-7B with CE loss on text reasoning traces without VT images for 2,000 steps; it starts from BAGEL-7B rather than ThinkMorph-7B because its training data overlaps ThinkMorph's original training distribution, so re-fitting the same data with CE loss would measure memorization rather than the text-only training signal. \methodnoise and \textbf{\method (Ours)} both run for 1,000 steps and differ only in whether the teacher context holds Gaussian noise or real VT images. \textbf{ThinkMorph} is the original model with full VT generation at inference, which serves as teacher and cost reference; we additionally evaluate it with VT generation disabled, the zero-shot text-only mode of Figure~\ref{fig:pilot}.

\paragraph{External Baselines.} GPT-4o, GPT-5, InternVL3.5-8B/38B~\citep{wang2025internvl3}, Qwen3-VL-8B/32B~\citep{bai2025qwen3}, Janus-Pro-7B~\citep{chen2025janus}, and Chameleon-7B~\citep{team2024chameleon}.

\paragraph{Training Details.} All \method variants are initialized from ThinkMorph-7B and trained on $8\times$H800 GPUs with FSDP (Hybrid Shard). Key hyperparameters: learning rate $1\mathrm{e}{-5}$ with cosine decay, EMA decay $0.995$, JSD $\beta{=}0.5$, top-$K{=}256$, token clip${=}0.05$, max completion length $1024$ tokens. The training data consists of $\sim$24K samples across 4 task categories. See Appendix~\ref{app:impl} for full details.

\subsection{Main Results}
\label{sec:main-results}

\begin{table}[!t]
\caption{Main results across 9 benchmarks. The three rows under \textit{Ours} report mean\,$\pm$\,standard deviation over 3 independent training and evaluation runs; every other row is a single evaluation. \textit{ThinkMorph (text-only)} is the same checkpoint with VT generation disabled at inference. \textbf{Bold}/\underline{underline}: best/second-best among open 7--8B models; larger and proprietary rows are for reference and excluded from ranking. VLM baselines on VStar use no external visual-search tools.}
\label{tab:main}
\centering

{\footnotesize
\setlength{\tabcolsep}{4pt}
\renewcommand{\arraystretch}{0.95}
\begin{tabularx}{\textwidth}{@{}l *{5}{>{\centering\arraybackslash}X}@{}}
\toprule
\textbf{Method} & \textbf{VSP}\,$\uparrow$ & \textbf{VisPuzzle}\,$\uparrow$ & \textbf{ChartQA}\,$\uparrow$ & \textbf{VStar}\,$\uparrow$ & \textbf{BLINK-J}\,$\uparrow$ \\
\midrule
\multicolumn{6}{@{}l}{\textit{Vision-Language Models (VLMs)}} \\
InternVL3.5-8B   & 8.17  & 34.75 & 76.26 & \underline{68.59} & \underline{71.33} \\
InternVL3.5-38B  & 20.16 & 36.50 & 80.44 & 76.96 & 80.67 \\
Qwen3-VL-8B     & 22.00 & 37.00 & \textbf{82.55} & \textbf{84.29} & 68.66 \\
Qwen3-VL-32B    & 43.50 & 50.75 & 83.52 & 84.80 & 70.66 \\
GPT-4o          & 33.50 & 43.75 & 76.34 & 61.78 & 72.67 \\
GPT-5           & 57.33 & 78.00 & 80.85 & 71.73 & 77.33 \\
\midrule
\multicolumn{6}{@{}l}{\textit{Unified Multimodal Models (UMMs)}} \\
Janus-Pro-7B    & 0.00  & 33.50 & 43.08 & 38.22 & 50.67 \\
Chameleon-7B    & 0.83  & 30.50 & 5.74  & 28.27 & 0.67  \\
BAGEL-7B        & 0.83  & 35.00 & 61.82 & 55.49 & 67.33 \\
ThinkMorph (text-only) & 73.27 & 76.25 & 76.80 & 64.92 & 68.00 \\
ThinkMorph      & \underline{75.83} & \underline{77.50} & 78.00 & 67.01 & 66.00 \\
\midrule
\multicolumn{6}{@{}l}{\textit{Ours}} \\
Text-only SFT   & 49.17\,{\scriptsize$\pm$1.06} & 63.50\,{\scriptsize$\pm$1.12} & \underline{81.66}\,{\scriptsize$\pm$1.16} & 56.02\,{\scriptsize$\pm$0.46} & 68.67\,{\scriptsize$\pm$0.72} \\
\methodnoise    & 50.83\,{\scriptsize$\pm$1.14} & 64.50\,{\scriptsize$\pm$0.53} & 73.77\,{\scriptsize$\pm$0.42} & 61.70\,{\scriptsize$\pm$0.52} & 68.66\,{\scriptsize$\pm$0.63} \\
\rowcolor{bestrow}
\textbf{\method (Ours)} & \textbf{85.83}\,{\scriptsize$\pm$0.89} & \textbf{86.00}\,{\scriptsize$\pm$0.82} & 78.79\,{\scriptsize$\pm$1.27} & 64.92\,{\scriptsize$\pm$0.70} & \textbf{77.33}\,{\scriptsize$\pm$0.63} \\
\midrule
\rowcolor{deltarow}
$\Delta$ \textit{vs.} ThinkMorph    & \textcolor{teal}{+10.00} & \textcolor{teal}{+8.50}  & \textcolor{teal}{+0.79}  & \textcolor{red!75!black}{$-$2.09} & \textcolor{teal}{+11.33} \\
\rowcolor{deltarow}
$\Delta$ \textit{vs.} Text-only SFT & \textcolor{teal}{+36.66} & \textcolor{teal}{+22.50} & \textcolor{red!75!black}{$-$2.87} & \textcolor{teal}{+8.90} & \textcolor{teal}{+8.66} \\
\bottomrule
\end{tabularx}%
}

\vspace{0.1em}

{\footnotesize
\setlength{\tabcolsep}{3pt}
\renewcommand{\arraystretch}{0.95}
\begin{tabularx}{\textwidth}{@{}l *{6}{>{\centering\arraybackslash}X}@{}}
\toprule
\textbf{Method} & \textbf{MMVP}\,$\uparrow$ & \textbf{SAT}\,$\uparrow$ & \textbf{BLINK}\,$\uparrow$ & \mbox{\textbf{CV-Bench}\,$\uparrow$} & \textbf{Avg}\,$\uparrow$ & \textbf{Lat.\,(s)}\,$\downarrow$ \\
\midrule
\multicolumn{7}{@{}l}{\textit{Vision-Language Models (VLMs)}} \\
InternVL3.5-8B  & 76.33 & 45.33 & 59.60 & \underline{81.99} & 58.04 & -- \\
InternVL3.5-38B & 80.33 & 49.33 & 62.65 & 85.96 & 63.67 & -- \\
Qwen3-VL-8B    & 77.66 & \textbf{54.00} & \textbf{69.43} & \textbf{85.63} & 64.58 & -- \\
Qwen3-VL-32B   & 80.00 & 76.00 & 66.22 & 87.51 & 71.44 & -- \\
GPT-4o         & 84.67 & 28.00 & 60.28 & 75.61 & 59.62 & -- \\
GPT-5          & 86.33 & 73.30 & 69.86 & 85.46 & 75.58 & -- \\
\midrule
\multicolumn{7}{@{}l}{\textit{Unified Multimodal Models (UMMs)}} \\
Janus-Pro-7B   & 63.33 & 22.00 & 38.51 & 67.83 & 39.68 & -- \\
Chameleon-7B   & 47.67 & 10.67 & 16.52 & 36.52 & 19.71 & -- \\
BAGEL-7B       & 70.33 & 44.67 & 47.66 & 76.03 & 51.02 & -- \\
ThinkMorph (text-only) & \textbf{79.00} & 51.33 & 59.50 & 80.07 & 69.90 & -- \\
ThinkMorph     & \underline{78.33} & \underline{52.67} & 59.49 & 80.86 & \underline{70.63} & 142.8 \\
\midrule
\multicolumn{7}{@{}l}{\textit{Ours}} \\
Text-only SFT  & 76.33\,{\scriptsize$\pm$0.37} & 46.63\,{\scriptsize$\pm$0.79} & 54.39\,{\scriptsize$\pm$0.67} & 77.37\,{\scriptsize$\pm$0.97} & 63.75 & 28.5 \\
\methodnoise   & 75.33\,{\scriptsize$\pm$0.36} & 48.00\,{\scriptsize$\pm$1.15} & 55.49\,{\scriptsize$\pm$0.85} & 79.09\,{\scriptsize$\pm$1.14} & 64.15 & \underline{14.9} \\
\rowcolor{bestrow}
\textbf{\method (Ours)} & 77.33\,{\scriptsize$\pm$0.68} & \textbf{54.00}\,{\scriptsize$\pm$1.21} & \underline{61.44}\,{\scriptsize$\pm$0.69} & 80.64\,{\scriptsize$\pm$0.51} & \textbf{74.03} & \textbf{10.0} \\
\midrule
\rowcolor{deltarow}
$\Delta$ \textit{vs.} ThinkMorph    & \textcolor{red!75!black}{$-$1.00} & \textcolor{teal}{+1.33} & \textcolor{teal}{+1.95} & \textcolor{red!75!black}{$-$0.22} & \textcolor{teal}{+3.40}  & \textcolor{teal}{$14.3\times$} \\
\rowcolor{deltarow}
$\Delta$ \textit{vs.} Text-only SFT & \textcolor{teal}{+1.00}  & \textcolor{teal}{+7.37} & \textcolor{teal}{+7.05} & \textcolor{teal}{+3.27}  & \textcolor{teal}{+10.28} & \textcolor{teal}{$2.9\times$} \\
\bottomrule
\end{tabularx}%
}
\end{table}

\paragraph{The student overtakes its teacher.} \method exceeds the VT teacher on 6 of 9 benchmarks, $+3.40$pp on average, with the margin concentrated where generation knowledge should matter most: BLINK-J $+11.3$, VSP $+10.0$, VisPuzzle $+8.5$. Two of the remaining differences carry no information. ChartQA ($+0.79$pp) and CV-Bench ($-0.22$pp) both fall inside the run-to-run spread of $\pm1.27$ and $\pm0.51$ in Table~\ref{tab:main}, so we do not read a direction into them.

\paragraph{Transfer is task-specific.} Against text-only SFT, \method gains $+10.28$pp on average, again concentrated on the VT-useful tasks: VSP $+36.7$, VisPuzzle $+22.5$, BLINK-J $+8.7$, SAT $+7.4$, BLINK $+7.1$. Among the VT-neutral tasks only V* moves comparably ($+8.9$); CV-Bench ($+3.3$) and MMVP ($+1.0$) move little and ChartQA drops by $2.87$pp, which fits chart understanding having little to gain from spatial generation knowledge. The same spatial advantage separates \method from same-scale VLMs: on VSP it reaches $85.8$ against $8.2$ for InternVL3.5-8B and $22.0$ for Qwen3-VL-8B, so generation training supplies spatial reasoning that dedicated VLMs at this scale do not have (per-task visualization in Appendix~\ref{app:transfer-efficiency-fig}).

\paragraph{Speed comes partly from brevity.} \method runs at 10.0s per sample, $14.3\times$ faster than the VT teacher (142.8s) and $2.9\times$ faster than text-only SFT (28.5s), under identical greedy decoding (temperature${=}0$, $1024$-token cap) on the same hardware. Part of the second factor is output length: \method emits $201$ tokens per sample on average against $411$ for SFT and $452$ for ThinkMorph's text stream (Appendix~\ref{app:token-stats}).

That compression deserves a comment, since a student trained against a verbose teacher might be expected to inherit its verbosity. It does not, because the teacher never writes anything: it only rescores the student's own on-policy trajectory, so length is set by the student's sampling and the only pressure on it comes from the per-position targets. Two properties of the objective push those targets short. The loss is pure JSD with no CE term, so nothing anchors the student to the long reference traces that the SFT baseline is fitted on. And the teacher has already been shown the VTs, so it has little reason to spend mass on narrating what the image contains; where the student opens such a narration, matching the teacher pulls mass toward committing to an answer instead. We did not isolate this second effect on its own, but it matches the token-level picture in Figure~\ref{fig:kl-heatmap}, where divergence sits on content rather than function words, and it matches compression being largest on VisPuzzle and BLINK-J, the tasks whose baseline traces carry the most narration.

\subsection{VT Information Quality Scaling}
\label{sec:vt-quality}

To separate generation knowledge from regularization, we vary the quality of the VT content in the teacher context and leave everything else fixed.

\begin{table}[h]
\caption{VT information quality scaling. \textbf{Avg\,(9)} is the mean over all 9 benchmarks of Table~\ref{tab:main}; the four columns to its left are representative tasks listed for readability.}
\label{tab:vt-quality}
\begin{center}
\small
\setlength{\tabcolsep}{4pt}
\begin{tabular}{lccccccc}
\toprule
\textbf{Variant} & \textbf{Teacher VT} & \textbf{VSP} & \textbf{VisPuz.} & \textbf{BLK-J} & \textbf{VStar} & \textbf{Avg\,(9)} & \textbf{$\Delta$\,SFT} \\
\midrule
ThinkMorph (text-only) & None (zero-shot) & 73.27 & 76.25 & 68.00 & 64.92 & 69.90 & +6.15 \\
Text-only SFT & None & 49.17 & 63.50 & 68.67 & 56.02 & 63.75 & -- \\
\methodnoise & Gaussian noise & 50.83 & 64.50 & 68.66 & 61.70 & 64.15 & +0.40 \\
\textbf{\method (Ours)} & Real VT & \textbf{85.83} & \textbf{86.00} & \textbf{77.33} & \textbf{64.92} & \textbf{74.03} & \textbf{+10.28} \\
\bottomrule
\end{tabular}
\end{center}
\end{table}

Real VT yields $+10.28$pp over SFT, against $+0.40$pp for noise. \methodnoise holds every non-VT element of the teacher context fixed, including the system prompt, question, reference intro and transition prompt, and swaps only the VT pixels for Gaussian noise of the same shape. Neither JSD regularization nor the surrounding privileged structure can therefore account for \method's gains.

\subsection{Analysis}
\label{sec:analysis}

\paragraph{Per-sample win/loss analysis.}
Aggregate accuracy hides how the two models disagree, so we compare their predictions sample by sample (Figure~\ref{fig:win-loss}). On the complex spatial benchmarks (BLINK-J net $+11.3$pp, VSP $+10.0$pp, VisPuzzle $+8.5$pp) \method wins far more samples than it loses. These tasks need multi-step spatial reasoning, the quality of ThinkMorph's generated VTs is bounded by its diffusion model, and Figure~\ref{fig:attention-layer} indicates that a rendered VT tends to dominate whatever reasoning follows it regardless of quality. Reasoning straight from the original input, \method avoids that dependence while keeping the generation pathway's knowledge in distribution (further examples in Appendix~\ref{app:case-study}).

ThinkMorph keeps a modest edge on three benchmarks (V* $-$2.1pp, MMVP $-$1.0pp, CV-Bench $-$0.2pp), all visually simpler, where its VTs come out cleaner. There the cost of losing explicit VT images outweighs the benefit of ignoring them: magnified object details and highlighted regions are cues that distributional distillation does not fully replace (failure cases in Appendix~\ref{app:failure-cases}).

\begin{figure}[!ht]
    \centering
    \begin{minipage}[c]{0.63\textwidth}
        \centering
        \includegraphics[width=\textwidth]{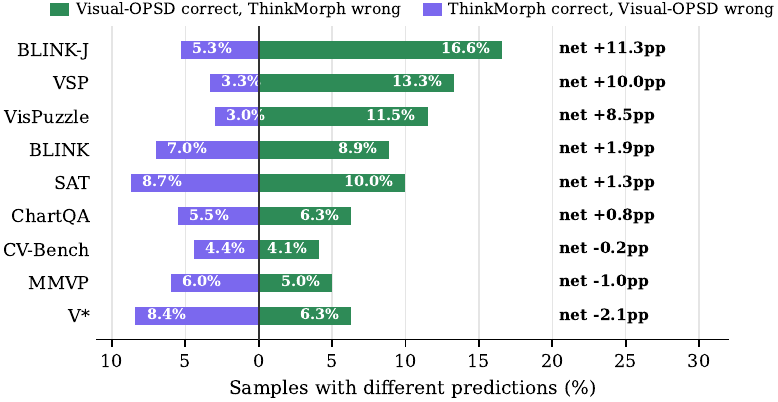}
    \end{minipage}\hfill
    \begin{minipage}[c]{0.33\textwidth}
        \caption{Per-sample win/loss between \method and ThinkMorph. Green: \method correct and ThinkMorph wrong; purple: the reverse. The deficits on ThinkMorph-leading benchmarks are small and near-symmetric.}
        \label{fig:win-loss}
    \end{minipage}
\end{figure}

\section{Conclusion}
\label{sec:conclusion}

We introduced \methodfull (\method), the first On-Policy Self-Distillation framework to operate across modalities within a single unified multimodal model. The visual generation pathway of a UMM turns out to hold reasoning knowledge that its rendered pixels do not deliver, and on-policy JSD moves that knowledge into the text understanding pathway with no architectural change. The \method student beats its generative teacher on 6 of 9 benchmarks ($+3.40$pp on average) at a $14.3\times$ speedup, and clears same-scale dedicated VLMs by a wide margin on spatial reasoning. The \methodnoise control ($+0.40$pp against $+10.28$pp) and the post-distillation KL analysis ($58.4\%$ against $3.5\%$ gap closure) place the transferred signal in the semantic content of the generation pathway rather than in regularization.

\clearpage

\bibliographystyle{plainnat}
\bibliography{references}

\clearpage

\beginappendix

\section{Related Work}
\label{app:related}

\paragraph{On-Policy Self-Distillation (OPSD) and its emerging family.}
\method belongs to the recently emerging family of on-policy self-distillation (OPSD) methods, in which a \emph{single} model instantiates both teacher and student by conditioning on different contexts, and token-level distillation is performed along the student's own on-policy trajectories. The principle was introduced by~\citet{zhao2026self} for text-only reasoning: the teacher conditions on the verified ground-truth solution while the student sees only the problem, and JSD on the student's rollouts transfers privileged reasoning knowledge without an external teacher model. \citet{zhang2026opsdl} extends OPSD to long-context language modeling, using a short-context self-teacher to denoise long-context generation. Closest to us in modality is Vision-OPD~\citep{yuan2026visionopd}, which applies OPSD within the visual modality: a crop-conditioned teacher supervises a full-image student to transfer fine-grained regional perception. \method differs from all prior OPSD instances in a critical way: the teacher--student information gap is \emph{cross-modal}, namely between the generation pathway (which has internalized visual reasoning via diffusion training) and the understanding pathway (which has not), rather than within text (OPSD), within long-vs-short context (OPSDL), or within visual crops (Vision-OPD). To our knowledge, \method is the first OPSD framework that bridges a generation--understanding gap within a unified multimodal architecture.

\paragraph{Knowledge Distillation and Privileged Information.}
Classical knowledge distillation~\citep{hinton2015distilling} transfers knowledge from a larger teacher to a smaller student. Born-Again Networks~\citep{furlanello2018born} showed that self-distillation between same-capacity networks can improve performance. CoT-style rationale distillation~\citep{hsieh2023distilling,magister2023teaching} extracts reasoning from larger models into smaller ones, while DPO-style methods~\citep{rafailov2024direct} align via preference signals. Learning Using Privileged Information (LUPI)~\citep{vapnik2015learning} formalizes the teacher--student gap in terms of information rather than capacity, and Zooming without Zooming~\citep{wei2026zooming} applies it to visual zooming with better data as the privileged signal. \method differs in kind from all of these: the teacher--student gap is neither in capacity nor in data quality, but in the \emph{modality of conditioning} within a single shared parameter set. The teacher's privileged information is the activation pattern induced by the generation pathway on its own VT outputs, that is, an internal-state asymmetry rather than an external data asymmetry.

\paragraph{Unified Multimodal Models.}
Recent work converges on architectures that handle visual understanding and generation within a single model: BAGEL~\citep{deng2025emerging} fuses a Qwen2.5 LLM backbone with a SigLIP vision encoder and a FLUX VAE; ThinkMorph~\citep{gu2026thinkmorph} adds interleaved visual chain-of-thought, generating intermediate images during reasoning; other notable UMMs include Chameleon~\citep{team2024chameleon}, Emu3~\citep{wang2024emu3}, Janus~\citep{wu2024janus} and Janus-Pro~\citep{chen2025janus}, and Transfusion~\citep{zhou2024transfusion}. Any UMM whose generation pathway can act as a privileged knowledge source is a candidate teacher for cross-modal OPSD.

\paragraph{Visual Chain-of-Thought.}
Visual CoT~\citep{rose2023visual} and Multimodal CoT~\citep{zhang2023multimodal} explore generating intermediate visual representations during reasoning. Visual Sketchpad~\citep{hu2024visual} uses code-generated sketches as reasoning aids. All of these methods retain visual generation at inference and therefore inherit its cost. \method instead distills the \emph{knowledge} from visual CoT into the text understanding pathway, removing the inference-time generation cost without losing reasoning accuracy.

\section{Proof of Theorem~\ref{thm:signal-noise}}
\label{app:proofs}

\begin{proof}
\textbf{(a) Forward-KL decomposition.} With the input, completion prefix, and teacher weights fixed, let $\bar p(v)=\mathbb{E}_V[p_V(v)]$. Splitting the log-ratio gives
\begin{align*}
\mathbb{E}_V[\KL(p_V\|q)]
&=\mathbb{E}_V\!\left[\sum_v p_V(v)\log\frac{p_V(v)}{\bar p(v)}\right]
  +\sum_v\mathbb{E}_V[p_V(v)]\log\frac{\bar p(v)}{q(v)}\\
&=\mathbb{E}_V[\KL(p_V\|\bar p)]+\KL(\bar p\|q).
\end{align*}
The first term is independent of $q$. Nonnegativity of KL, with equality only when $q=\bar p$, proves the unique minimizer. Boundary distributions with $q(v)=0$ have infinite forward KL because $\bar p(v)>0$.

\textbf{(b) JSD lower bound.} Write $m_V=(p_V+q)/2$ and $\bar m=(\bar p+q)/2=\mathbb{E}_V[m_V]$. Joint convexity of KL yields
\begin{align*}
\mathbb{E}_V[\mathrm{JSD}_{1/2}(p_V\|q)]
&=\tfrac12\mathbb{E}_V[\KL(p_V\|m_V)]
  +\tfrac12\mathbb{E}_V[\KL(q\|m_V)]\\
&\geq\tfrac12\KL(\bar p\|\bar m)+\tfrac12\KL(q\|\bar m)\\
&=\mathrm{JSD}_{1/2}(\bar p\|q).
\end{align*}
This proves Eq.~\ref{eq:jsd-noise}.
\end{proof}

\section{Implementation Details}
\label{app:impl}

\paragraph{Model Architecture.} ThinkMorph-7B, the checkpoint \method starts from, consists of: Qwen2.5 LLM backbone with Qwen2MoTDecoderLayer ($\sim$7B parameters), SigLIP-so400m-14-980 NaViT vision encoder, and FLUX VAE for latent image encoding. Total trainable parameters: $\sim$1820M (99.9\% of total 1822M).

\paragraph{Training Infrastructure.} 8$\times$NVIDIA H800 80GB GPUs with FSDP (HYBRID\_SHARD strategy). Activation checkpointing on all Qwen2MoTDecoderLayer modules. Optimizer state CPU offloading during sampling and forward/backward to manage memory constraints.

\paragraph{Hyperparameters.} Table~\ref{tab:hyperparameters} lists the training configuration.
\begin{table}[htbp]
\caption{Training hyperparameters for \method.}
\label{tab:hyperparameters}
\begin{center}
\small
\begin{tabular}{ll|ll}
\toprule
\textbf{Parameter} & \textbf{Value} & \textbf{Parameter} & \textbf{Value} \\
\midrule
Learning rate & 1e-5 & JSD $\beta$ & 0.5 \\
Min learning rate & 1e-7 & JSD temperature & 1.0 \\
LR scheduler & Cosine & JSD top-$K$ & 256 \\
Warmup steps & 200 & Token clip & 0.05 \\
Total steps & 1000 & Loss & Pure JSD \\
AdamW $(\beta_1, \beta_2)$ & (0.9, 0.95) & On-policy sampling & Yes \\
AdamW $\epsilon$ & 1e-15 & EMA decay & 0.995 \\
Max grad norm & 1.0 & Teacher mode & EMA \\
Gradient accum. steps & 2 & Max completion tokens & 1024 \\
Max forward tokens & 10240 & Image skips & 1 \\
\bottomrule
\end{tabular}
\end{center}
\end{table}

\paragraph{Data.} Training data consists of $\sim$24,990 samples across four task categories, identical to the ThinkMorph training distribution~\citep{gu2026thinkmorph}: Visual Search (6,990), Spatial Navigation (6,000), Jigsaw Assembly (6,000), Chart Refocus (6,000). Each sample contains a problem image, question, interleaved textual reasoning traces with VT images, and the answer. Training and evaluation data are disjoint: training samples are drawn from designated training splits, while all 9 evaluation benchmarks use their respective held-out test sets. No evaluation benchmark images or questions appear in the training data. The 1,000 samples used for the KL diagnostic of Section~\ref{sec:kl-diag} and Appendix~\ref{app:kl-closing} are drawn from the same four categories but are held out from training. Images are processed with NaViT-style patching: stride 14, max size 980, min size 378, max pixels 2,007,040.

\paragraph{Visual token budget.} The two caps interact as follows. A ViT grid is limited to 70 patches per side, so one image contributes at most $70\times70 = 4{,}900$ tokens; the max-pixel setting of 2,007,040 is an area cap that the side cap makes unreachable, because a $70\times70$ grid at stride 14 covers at most $960{,}400$ pixels. Dividing the area cap by $14^2$ therefore overstates the per-image cost by more than a factor of two. The student's forward pass holds the system prompt, the problem image and the question ahead of a completion capped at 1,024 tokens, so it stays under ${\approx}6{,}000$ tokens and never reaches the 10,240 budget.

\paragraph{Memory Optimization.} Key optimizations to fit within 8$\times$80GB: (1)~AdamW state CPU offload during sampling ($\sim$14.5 GiB freed); (2)~FSDP precision patch preventing FP32 upcast during parameter unshard (58GB$\to$29GB); (3)~Aggressive tensor release between student/teacher forward passes; (4)~Sequence length limiting (max 10,240 tokens) with automatic completion truncation.

\paragraph{Training Loop.} Algorithm~\ref{alg:opsd-vpd} summarizes one optimization step of \method, including on-policy sampling with image-skip injection, dual teacher/student forward passes, the per-token JSD objective, and the EMA teacher update.

\begin{algorithm}[!ht]
\caption{\method Training Loop}
\label{alg:opsd-vpd}
\begin{algorithmic}[1]
\REQUIRE Model $M_\theta$, EMA teacher $M_{\bar\theta}$, dataset $\mathcal{D}$
\FOR{each training step}
    \STATE Sample raw data $x \sim \mathcal{D}$ (problem image, question, reference trace with VT)
    \STATE \textbf{On-policy sampling}: Generate completion $c \sim M_\theta(\cdot \mid \text{img}, \text{question})$
    \STATE \quad On \texttt{<image\_start>}, append \texttt{<image\_end>} to the context and resume text sampling
    \STATE \quad (the injected token is excluded from $c$, so it carries no loss term)
    \STATE Build student batch: $[\text{sys}, \text{ViT}(\text{img}), \text{q}, c]$
    \STATE Build teacher batch: $[\text{sys}, \text{ViT}(\text{img}), \text{q}, \text{intro}, (\text{ViT}(\text{VT}_i))^+, \text{trans}, c]$ using $M_{\bar\theta}$
    \STATE \textbf{Dual forward}: Compute $\text{logits}_S$ (with grad), $\text{logits}_T$ (no grad)
    \STATE Restrict to teacher top-$K$, renormalize both sides, accumulate clipped per-entry JSD (Eq.~\ref{eq:vpd-loss-impl})
    \STATE Backward, optimizer step, EMA update: $\bar\theta \leftarrow \alpha \bar\theta + (1-\alpha)\theta$
\ENDFOR
\end{algorithmic}
\end{algorithm}

\section{Initialization Rationale}
\label{app:init-rationale}

\method variants (\method, \methodnoise) are initialized from ThinkMorph-7B because the teacher pathway requires interleaved visual generation capability, which is present only in ThinkMorph and not in the base BAGEL model. \method does not optimize a CE loss on the training data; it distills distributional knowledge via JSD on on-policy completions, so it does not directly memorize training examples. Text-only SFT, in contrast, is initialized from BAGEL-7B: it trains with standard CE on text reasoning traces drawn from the same distribution used to train ThinkMorph, so initializing from ThinkMorph would amount to re-fitting on already-seen data, producing overfitting rather than a fair assessment of the text-only training signal. The $+10.28$pp gain of \method over Text-only SFT reflects two factors: (1)~the ThinkMorph initialization, which already embeds generation-trained representations, and (2)~the \method distillation objective. The \methodnoise control isolates factor~(2): it uses the same ThinkMorph initialization and the same objective structure but replaces real VT pixels with semantic-free Gaussian noise, and gains only $+0.40$pp over SFT. The transfer of the generation-pathway signal therefore requires the distillation objective applied to semantically meaningful VT content, not the ThinkMorph initialization on its own. Table~\ref{tab:main} additionally reports ThinkMorph in text-only inference mode ($69.90$) as an untrained reference for that same initialization. \methodnoise sits below it, so distilling against a semantically empty teacher does not merely fail to help but costs accuracy relative to the initialization it starts from, whereas \method exceeds the reference by $+4.13$pp.

\section{Teacher Context Prompts}
\label{app:prompts}

The teacher context $\mathcal{C}_T$ (Eq.~\ref{eq:ctx-teacher}) uses two framing prompts to separate the privileged visual reasoning trace from the student's own completion. Their full text is given below.

\paragraph{Reference Introduction (\textnormal{ref\_intro}).} Placed immediately before the privileged VT images. Because the privileged channel is visual-only, the prompt explicitly refers to images rather than to a textual reasoning trace:
\begin{quote}
\small\itshape
``The following images are privileged visual references that depict the intermediate visual thoughts on the path to the correct answer. Use them silently as grounding context; do not describe or echo them.''
\end{quote}

\paragraph{Transition Prompt (\textnormal{transition}).} Placed after the privileged VT images and before the completion tokens on which loss is computed:
\begin{quote}
\small\itshape
``Now, using your own independent reasoning, answer the problem above. Think step by step.''
\end{quote}

Both prompts carry zero loss weight during training; they serve only to structure the teacher's KV cache so that the privileged VT images are absorbed as visual context rather than directly copied or described in the completion.

\section{Per-Token KL Divergence Analysis}
\label{app:kl-heatmap}

\begin{figure}[h]
\centering
\includegraphics[width=\linewidth]{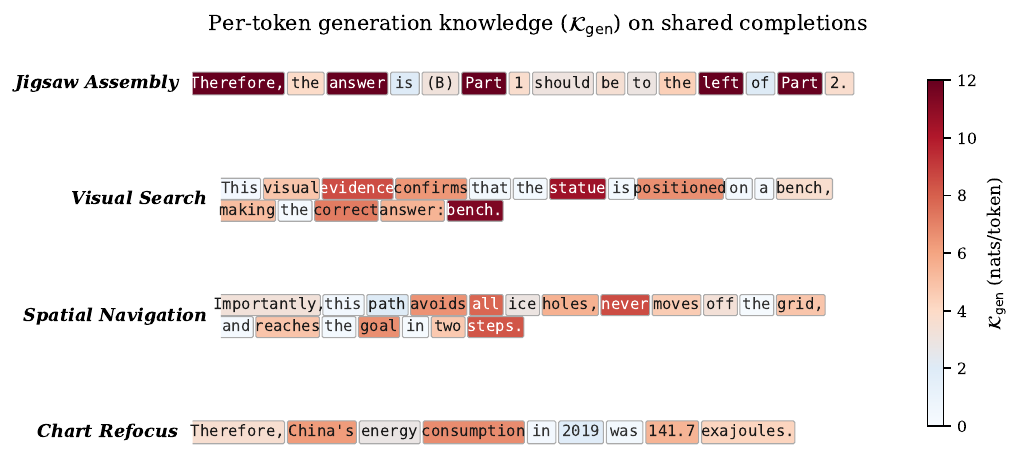}
\caption{Per-token $\mathcal{K}_{\text{gen}}$ on representative completions drawn from the held-out set of each task category. Generation knowledge concentrates on informationally critical tokens such as spatial labels (\texttt{Part}, \texttt{left}), object references (\texttt{statue}, \texttt{bench}), navigation decisions (\texttt{goal}, \texttt{avoids}), and quantitative values (\texttt{141.7}, \texttt{2019}), while function words carry near-zero divergence.}
\label{fig:kl-heatmap}
\end{figure}

The per-token KL analysis reveals that generation knowledge is \emph{non-uniform}: it concentrates on tokens encoding spatial relations, quantities, and visual-grounded answers, while function words and syntactic connectives exhibit near-zero KL. This pattern is distinct from the high-entropy ``forking tokens'' observed in reinforcement learning~\citep{wang2025beyond}, where connectives represent trajectory-level decision points. Our cross-context KL instead measures how much additional \emph{perceptual evidence} the VT context contributes at each position: VT informs \emph{what} the answer is (content tokens) rather than \emph{how} to express it (syntactic structure), producing the observed content-specific divergence pattern. Since teacher and student share both the model weights and the completion tokens, the measured divergence is attributable solely to the privileged VT context.

\section{Output Token Statistics}
\label{app:token-stats}

\begin{table}[h]
\caption{Mean output tokens per sample. All methods use greedy decoding (temperature=0, max 1024 tokens). ThinkMorph token counts exclude diffusion steps (text tokens only).}
\label{tab:token-stats}
\begin{center}
\small
\begin{tabular}{lccc}
\toprule
\textbf{Benchmark} & \textbf{\method} & \textbf{ThinkMorph} & \textbf{Text-only SFT} \\
\midrule
BLINK       & 258.2 & 522.9 & 439.4 \\
BLINK-J     & 287.9 & 604.6 & 552.1 \\
ChartQA (h) & 181.8 & 397.7 & 360.6 \\
ChartQA (v) & 194.7 & 416.8 & 379.6 \\
CV-Bench-2D & 148.3 & 304.1 & 307.2 \\
CV-Bench-3D & 171.1 & 383.0 & 349.8 \\
MMVP        & 165.7 & 407.2 & 371.9 \\
SAT         & 189.2 & 404.5 & 369.4 \\
VisPuzzle   & 263.4 & 768.1 & 701.5 \\
V*          & 153.9 & 308.4 & 281.6 \\
\midrule
\textbf{Average} & \textbf{201.4} & 451.7 & 411.3 \\
\bottomrule
\end{tabular}
\end{center}
\end{table}

\method produces shorter outputs on every benchmark in Table~\ref{tab:token-stats}, whose 10 rows split ChartQA and CV-Bench into their two variants. On average \method generates 201.4 tokens per sample, $2.0\times$ fewer than SFT (411.3) and $2.2\times$ fewer than ThinkMorph (451.7, text tokens only). The compression is largest on VisPuzzle (263.4 against 701.5 for SFT, $2.7\times$) and BLINK-J (287.9 against 604.6 for ThinkMorph, $2.1\times$), both complex spatial tasks where the baselines write lengthy ``I observe that\ldots'' narrations. Even where accuracy is comparable across methods, as on CV-Bench-2D, \method emits ${\sim}2\times$ fewer tokens, so the conciseness is a general property of the distillation rather than an artifact of particular tasks.

\section{Per-Task Knowledge Transfer and Inference Efficiency}
\label{app:transfer-efficiency-fig}

Figure~\ref{fig:transfer_efficiency} visualizes the per-task gain pattern and the latency profile discussed in Section~\ref{sec:main-results}. Spatial reasoning benchmarks dominate the mean $+10.28$pp gain over Text-only SFT while ChartQA shows essentially no transfer, and \method at 10.0s per sample is $14.3\times$ faster than the VT teacher and $2.9\times$ faster than text-only SFT, the latter factor following from its ${\sim}2\times$ shorter outputs (Appendix~\ref{app:token-stats}).

\begin{figure}[h]
    \centering
    \begin{subfigure}[t]{0.48\textwidth}
        \centering
        \includegraphics[width=\textwidth]{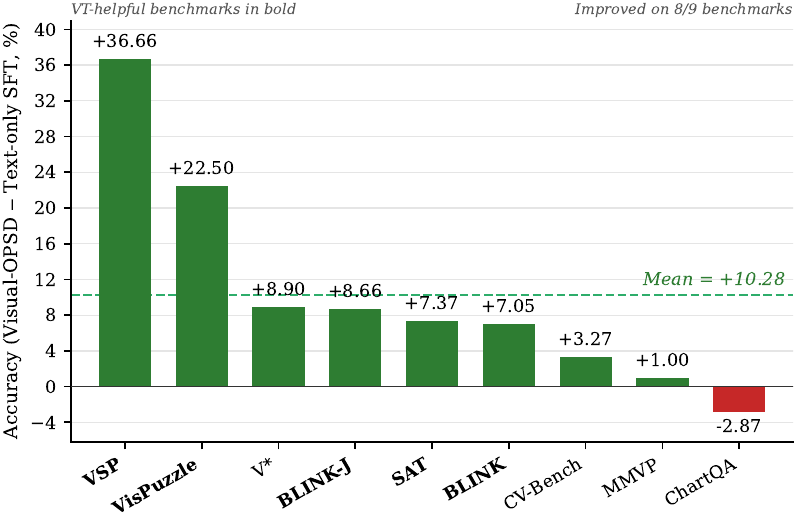}
        \caption{}
        \label{fig:delta_sft}
    \end{subfigure}
    \hfill
    \begin{subfigure}[t]{0.48\textwidth}
        \centering
        \includegraphics[width=\textwidth]{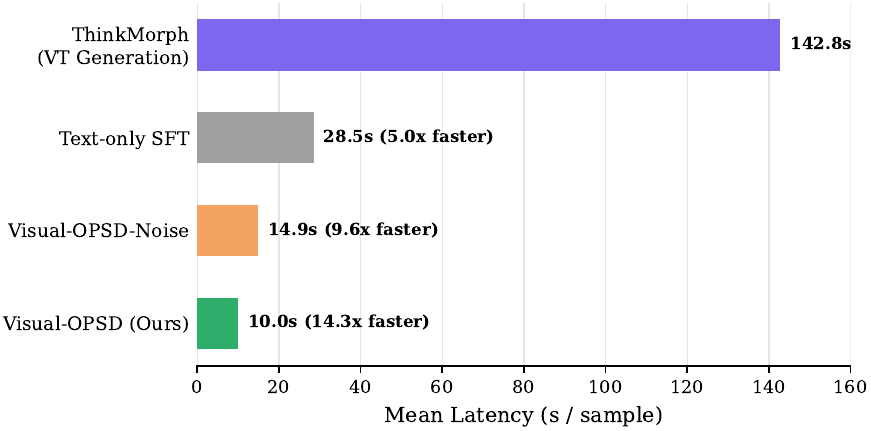}
        \caption{}
        \label{fig:latency}
    \end{subfigure}
    \caption{%
        \textbf{Task-specific knowledge transfer and inference efficiency.}
        \textbf{(a)}~Generation knowledge transfers selectively: spatial reasoning tasks benefit most (mean $\Delta$=+10.28pp over Text-only SFT), while chart understanding shows minimal change.
        \textbf{(b)}~\method is faster than both the VT teacher (14.3$\times$) and text-only SFT (2.9$\times$), suggesting distillation produces more concise reasoning.
    }
    \label{fig:transfer_efficiency}
\end{figure}



\section{Knowledge Transfer Verification: Post-Distillation KL Analysis}
\label{app:kl-closing}

The KL diagnostic in Section~\ref{sec:kl-diag} establishes that a distributional gap $\mathcal{K}_{\text{gen}}$ exists between teacher (with VT context) and student (without) \emph{before} any distillation training. A natural question is whether \method successfully closes this gap, which would provide direct evidence that generation knowledge has been internalized into the understanding pathway.

\paragraph{Protocol.} We re-run the identical KL diagnostic (Eq.~\ref{eq:gen-knowledge}) on the same 1,000 held-out samples, replacing the student with each trained checkpoint. These samples never appear in the training data, so the numbers below measure generalization of the distillation objective rather than its fit on the optimized examples. For each variant, we compute:
\begin{equation}
    \Delta\mathcal{K} = 1 - \frac{\mathcal{K}_{\text{gen}}^{\text{post}}}{\mathcal{K}_{\text{gen}}^{\text{pre}}} \in [0, 1]
\end{equation}
where $\Delta\mathcal{K}$ represents the fraction of the teacher--student distributional gap that has been closed by training. A value of 1 indicates perfect knowledge internalization; 0 indicates no transfer.

\begin{table}[h]
\caption{Post-distillation KL gap analysis on 1,000 held-out samples. We measure the teacher--student distributional gap $\mathcal{K}_{\text{gen}}$ (nats/token) before and after training for each variant (both initialized from ThinkMorph-7B). \method closes 58.4\% of the gap on average, with the largest reductions on spatial reasoning tasks, while \methodnoise closes under 4\%.}
\label{tab:kl-closing}
\begin{center}
\small
\begin{tabular}{lccc}
\toprule
\textbf{Task Category} & $\boldsymbol{\mathcal{K}_{\text{gen}}^{\text{pre}}}$ & \multicolumn{2}{c}{$\boldsymbol{\mathcal{K}_{\text{gen}}^{\text{post}}}$ \textbf{(nats/token)} [$\boldsymbol{\Delta\mathcal{K}}$\textbf{\%}]} \\
\cmidrule(lr){3-4}
& \textbf{(frozen)} & \textbf{\methodnoise} & \textbf{\method (Ours)} \\
\midrule
Visual Search & 4.23 & 4.08 \scriptsize[3.5\%] & 1.72 \scriptsize[\textbf{59.3\%}] \\
Spatial Navigation & 3.96 & 3.82 \scriptsize[3.5\%] & 1.48 \scriptsize[\textbf{62.6\%}] \\
Chart Refocus & 3.51 & 3.39 \scriptsize[3.4\%] & 2.14 \scriptsize[\textbf{39.0\%}] \\
Jigsaw Assembly & 6.84 & 6.61 \scriptsize[3.4\%] & 2.37 \scriptsize[\textbf{65.4\%}] \\
\midrule
\textbf{Overall} & \textbf{4.64} & 4.48 \scriptsize[3.5\%] & \textbf{1.93} \scriptsize[\textbf{58.4\%}] \\
\bottomrule
\end{tabular}
\end{center}
\end{table}

\paragraph{Results and interpretation.}

The \method student closes 58.4\% of the teacher--student gap on average. On samples it was never trained on, its predictions, made \emph{without} any VT context, move most of the way toward what the model would produce \emph{if} it had observed the full sequence of privileged VT images, which is what transfer of generation-pathway knowledge into the understanding pathway should look like.

The pattern tracks the accuracy gains of Table~\ref{tab:main}. Jigsaw Assembly closes the most ($65.4\%$) and carries the largest improvement over the teacher (BLINK-J $+11.3$pp); Chart Refocus closes the least ($39.0\%$) and is the one category where distillation fails to beat text-only SFT (ChartQA $-2.87$pp). \methodnoise closes $3.5\%$, which is attributable to mild EMA-based regularization rather than transfer. Aligning with the VT-conditioned teacher therefore requires semantically meaningful VT content; Gaussian noise in the teacher context supplies almost no learning signal.

\section{JSD Hyperparameter Sensitivity}
\label{app:jsd-ablation}

We ablate the two key JSD loss hyperparameters (top-$K$ vocabulary truncation and per-token clipping threshold) to assess \method's robustness. All variants are trained for 1,000 steps with other hyperparameters held at their default values (Section~\ref{app:impl}). We report average accuracy across 3 representative benchmarks (VSP, VisPuzzle, BLINK-J).

\begin{table}[h]
\caption{JSD hyperparameter sensitivity. \textbf{Left}: varying top-$K$ with clip=0.05. \textbf{Right}: varying clip threshold with $K$=256. The default configuration ($K$=256, clip=0.05) achieves the best overall performance, but \method is robust across a wide range of settings.}
\label{tab:jsd-ablation}
\begin{center}
\small
\begin{tabular}{lcccc|lcccc}
\toprule
\textbf{Top-$K$} & \textbf{VSP} & \textbf{VisPuz.} & \textbf{BLK-J} & \textbf{Avg$_3$} & \textbf{Clip} & \textbf{VSP} & \textbf{VisPuz.} & \textbf{BLK-J} & \textbf{Avg$_3$} \\
\midrule
64  & 81.67 & 83.50 & 73.33 & 79.50 & None & 82.50 & 83.00 & 74.00 & 79.83 \\
128 & 84.17 & 85.00 & 76.00 & 81.72 & 0.01 & 83.33 & 84.50 & 75.33 & 81.05 \\
\textbf{256} & \textbf{85.83} & \textbf{86.00} & \textbf{77.33} & \textbf{83.05} & \textbf{0.05} & \textbf{85.83} & \textbf{86.00} & \textbf{77.33} & \textbf{83.05} \\
512 & 84.17 & 85.50 & 76.00 & 81.89 & 0.10 & 84.17 & 85.00 & 75.33 & 81.50 \\
\bottomrule
\end{tabular}
\end{center}
\end{table}

Table~\ref{tab:jsd-ablation} shows that \method is robust to hyperparameter choices: all configurations substantially outperform the Text-only SFT baseline (Avg$_3$=60.45). Performance degrades modestly with very small $K$ (64) due to loss of distributional information in the long tail, or without clipping (None) where noisy style-token gradients introduce variance. The default $K$=256 and clip=0.05 achieve the best balance between capturing sufficient distributional information and suppressing noise.

\section{Case Study: VT Interference in Generative Reasoning}
\label{app:case-study}

We present qualitative examples where ThinkMorph's generated VT images mislead subsequent reasoning, while \method avoids these failure modes. For each case we show: the input with question, the VT image generated by ThinkMorph, key reasoning excerpts from both models, and final answers.

\begin{figure*}[!ht]
\centering
\begin{minipage}[t]{0.60\textwidth}
    \vspace{0pt}
    \begin{minipage}[t]{0.24\textwidth}\centering
        \includegraphics[width=\textwidth]{figures/case1_input.jpg}\\[-1pt]
        {\scriptsize (a) Input}
    \end{minipage}\hfill
    \begin{minipage}[t]{0.24\textwidth}\centering
        \includegraphics[width=\textwidth]{figures/case1_optA.jpg}\\[-1pt]
        {\scriptsize (b) Option A}
    \end{minipage}\hfill
    \begin{minipage}[t]{0.24\textwidth}\centering
        \includegraphics[width=\textwidth]{figures/case1_optB.jpg}\\[-1pt]
        {\scriptsize (c) Option B \textcolor{green!50!black}{\ding{51}}}
    \end{minipage}\hfill
    \begin{minipage}[t]{0.24\textwidth}\centering
        \includegraphics[width=\textwidth]{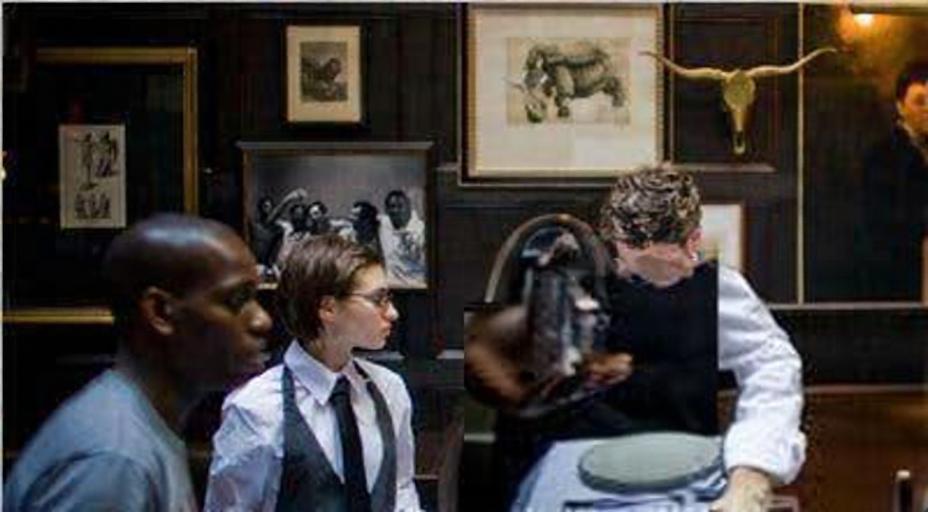}\\[-1pt]
        {\scriptsize (d) ThinkMorph VT}
    \end{minipage}
\end{minipage}\hfill
\begin{minipage}[t]{0.38\textwidth}
    \vspace{0pt}
    {\small\textbf{Question:} Which image is the missing part in the first image?}
    \vspace{4pt}

    {\small\textbf{Ground truth:} B (the third image).}
\end{minipage}

\vspace{6pt}
\hrule
\vspace{4pt}

\begin{minipage}[t]{0.48\textwidth}
    \colorbox{red!6}{\parbox{\dimexpr\textwidth-2\fboxsep}{%
    {\small\textbf{ThinkMorph} \textcolor{red}{\ding{55}} $\to$ \textbf{A}}\\[2pt]
    \scriptsize
    \textbf{Round\,0:} ``The arm and vest continue naturally from the first image\ldots the contours of the arm line up with the visible body.'' $\to$ commits to A.\\[2pt]
    \textbf{VT generation:} produces image (d) attempting to verify the completed puzzle.\\[2pt]
    \textbf{Round\,1} (conditioned on VT): ``The second image seamlessly fits into the missing lower right quadrant\ldots the white shirt and black vest connect smoothly.''\\[2pt]
    $\Rightarrow$ \textbf{VT reinforces initial error.}
    }}
\end{minipage}\hfill
\begin{minipage}[t]{0.48\textwidth}
    \colorbox{green!6}{\parbox{\dimexpr\textwidth-2\fboxsep}{%
    {\small\textbf{\method} \textcolor{green!50!black}{\ding{51}} $\to$ \textbf{B}}\\[2pt]
    \scriptsize
    ``The third image, with the blue apron and matching white shirt, fits perfectly into this scene: the apron's straps and the shirt's sleeve connect seamlessly with the upper body.''\\[4pt]
    ``If I were to use the second image, the table and objects would not fit as smoothly.''\\[4pt]
    $\Rightarrow$ \textbf{Directly compares both options against original input.}
    }}
\end{minipage}
\caption{\textbf{Case 1: Self-reinforcing VT confirmation.} The VT image~(d) is generated \emph{after} ThinkMorph has already committed to Option~A. The generated image visually reinforces the initial error, creating a self-reinforcing feedback loop. \method avoids this loop by reasoning directly from the original input.}
\label{fig:case1}
\end{figure*}

\begin{figure*}[!ht]
\centering
\begin{minipage}[t]{0.60\textwidth}
    \vspace{0pt}
    \begin{minipage}[t]{0.24\textwidth}\centering
        \includegraphics[width=\textwidth]{figures/case2_input.jpg}\\[-1pt]
        {\scriptsize (a) Input}
    \end{minipage}\hfill
    \begin{minipage}[t]{0.24\textwidth}\centering
        \includegraphics[width=\textwidth]{figures/case2_optA.jpg}\\[-1pt]
        {\scriptsize (b) Option A}
    \end{minipage}\hfill
    \begin{minipage}[t]{0.24\textwidth}\centering
        \includegraphics[width=\textwidth]{figures/case2_optB.jpg}\\[-1pt]
        {\scriptsize (c) Option B \textcolor{green!50!black}{\ding{51}}}
    \end{minipage}\hfill
    \begin{minipage}[t]{0.24\textwidth}\centering
        \includegraphics[width=\textwidth]{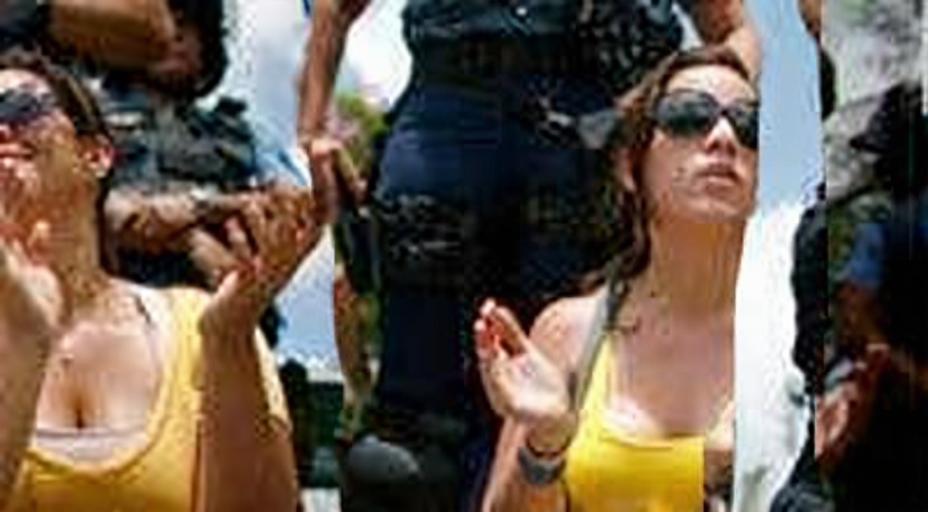}\\[-1pt]
        {\scriptsize (d) ThinkMorph VT}
    \end{minipage}
\end{minipage}\hfill
\begin{minipage}[t]{0.38\textwidth}
    \vspace{0pt}
    {\small\textbf{Question:} Which image is the missing part in the first image?}
    \vspace{4pt}

    {\small\textbf{Ground truth:} B (the third image).}
\end{minipage}

\vspace{6pt}
\hrule
\vspace{4pt}

\begin{minipage}[t]{0.48\textwidth}
    \colorbox{red!6}{\parbox{\dimexpr\textwidth-2\fboxsep}{%
    {\small\textbf{ThinkMorph} \textcolor{red}{\ding{55}} $\to$ \textbf{A}}\\[2pt]
    \scriptsize
    \textbf{Round\,0:} ``The left edge of this image aligns with the right edge of the first image, as both show the woman in the tank top at a matching angle.'' $\to$ commits to A.\\[2pt]
    \textbf{VT generation:} produces image (d) with visible blending artifacts at patch boundaries.\\[2pt]
    \textbf{Round\,1} (conditioned on VT): Confirms A despite spatial misalignment; pixel-level noise in (d) obscures the boundary discontinuities that would reveal incorrect placement.\\[2pt]
    $\Rightarrow$ \textbf{VT noise masks spatial errors.}
    }}
\end{minipage}\hfill
\begin{minipage}[t]{0.48\textwidth}
    \colorbox{green!6}{\parbox{\dimexpr\textwidth-2\fboxsep}{%
    {\small\textbf{\method} \textcolor{green!50!black}{\ding{51}} $\to$ \textbf{B}}\\[2pt]
    \scriptsize
    ``The third image contains the woman in yellow, the other person, and the officers in perfect alignment. The officers' arms, uniforms, and the sky all connect seamlessly.''\\[4pt]
    ``By focusing on the visual flow, especially the alignment of the woman's hands, the officers' arms, and the background, I see that the third image is the only one that fits perfectly.''\\[4pt]
    $\Rightarrow$ \textbf{Correct spatial continuity without noisy VT.}
    }}
\end{minipage}
\caption{\textbf{Case 2: Pixel-level VT noise obscures spatial cues.} The VT image~(d) introduces blending artifacts at patch boundaries, masking the visual discontinuities that would reveal incorrect placement. \method identifies correct spatial continuity directly from the original.}
\label{fig:case2}
\end{figure*}

\begin{figure*}[!ht]
\centering
\begin{minipage}[t]{0.60\textwidth}
    \vspace{0pt}
    \begin{minipage}[t]{0.45\textwidth}\centering
        \includegraphics[width=\textwidth]{figures/case3_input.png}\\[-1pt]
        {\scriptsize (a) Input image}
    \end{minipage}\hfill
    \begin{minipage}[t]{0.45\textwidth}\centering
        \includegraphics[width=\textwidth]{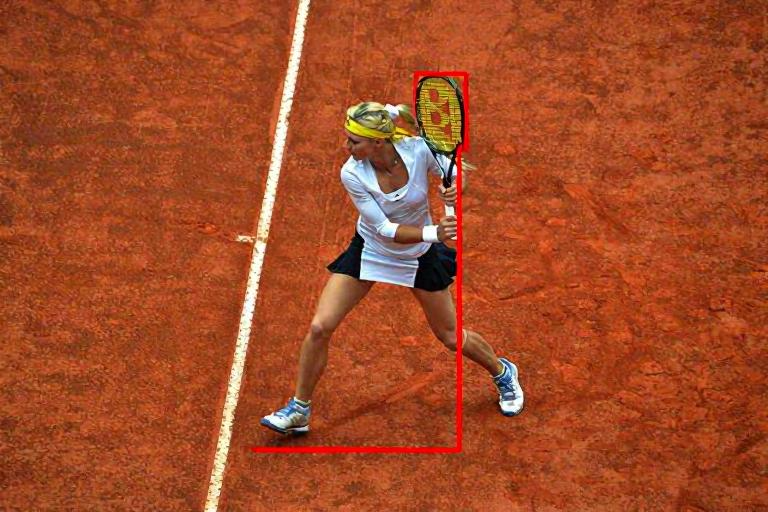}\\[-1pt]
        {\scriptsize (b) ThinkMorph VT}
    \end{minipage}\hfill
\end{minipage}\hfill
\begin{minipage}[t]{0.38\textwidth}
    \vspace{0pt}
    {\small\textbf{Question:} Is she about to swing the racket towards her left or right?}
    \vspace{4pt}

    {\small\textbf{Options:} (A)~right \quad (B)~left}
    \vspace{2pt}

    {\small\textbf{Ground truth:} A (right).}
\end{minipage}

\vspace{6pt}
\hrule
\vspace{4pt}

\begin{minipage}[t]{0.48\textwidth}
    \colorbox{red!6}{\parbox{\dimexpr\textwidth-2\fboxsep}{%
    {\small\textbf{ThinkMorph} \textcolor{red}{\ding{55}} $\to$ \textbf{B (left)}}\\[2pt]
    \scriptsize
    \textbf{Round\,0:} ``Her right arm is extended with the racket\ldots The racket is currently facing towards her right side.'' Despite noting rightward cues, concludes ``she is preparing to swing towards her left.''\\[2pt]
    \textbf{VT generation:} produces image (b) highlighting racket with bounding box, isolating it from full-body context.\\[2pt]
    \textbf{Round\,1} (conditioned on VT): ``The racket is pointed diagonally towards her left side, aligning with the direction of her body movement.'' Doubles down on wrong answer.\\[2pt]
    $\Rightarrow$ \textbf{VT annotation strips spatial context.}
    }}
\end{minipage}\hfill
\begin{minipage}[t]{0.48\textwidth}
    \colorbox{green!6}{\parbox{\dimexpr\textwidth-2\fboxsep}{%
    {\small\textbf{\method} \textcolor{green!50!black}{\ding{51}} $\to$ \textbf{A (right)}}\\[2pt]
    \scriptsize
    ``Her right arm is bent and gripping the racket\ldots The woman's head and torso are oriented towards the right side\ldots her right leg is extended as if preparing to move.''\\[4pt]
    $\Rightarrow$ \textbf{Integrates full-body posture holistically:} foot positioning, torso orientation, and shoulder angle jointly determine the correct swing direction.
    }}
\end{minipage}
\caption{\textbf{Case 3: VT annotation strips spatial context.} ThinkMorph's VT image~(b) highlights the racket via bounding box but isolates it from full-body context. Post-generation attention shifts to the annotated region (\texttt{txt1}$\to$\texttt{img1} dominance), causing reasoning from a spatially impoverished representation. \method integrates holistic body cues correctly.}
\label{fig:case3}
\end{figure*}

\begin{figure*}[!ht]
\centering
\begin{minipage}[t]{0.60\textwidth}
    \vspace{0pt}
    \begin{minipage}[t]{0.30\textwidth}\centering
        \includegraphics[width=\textwidth]{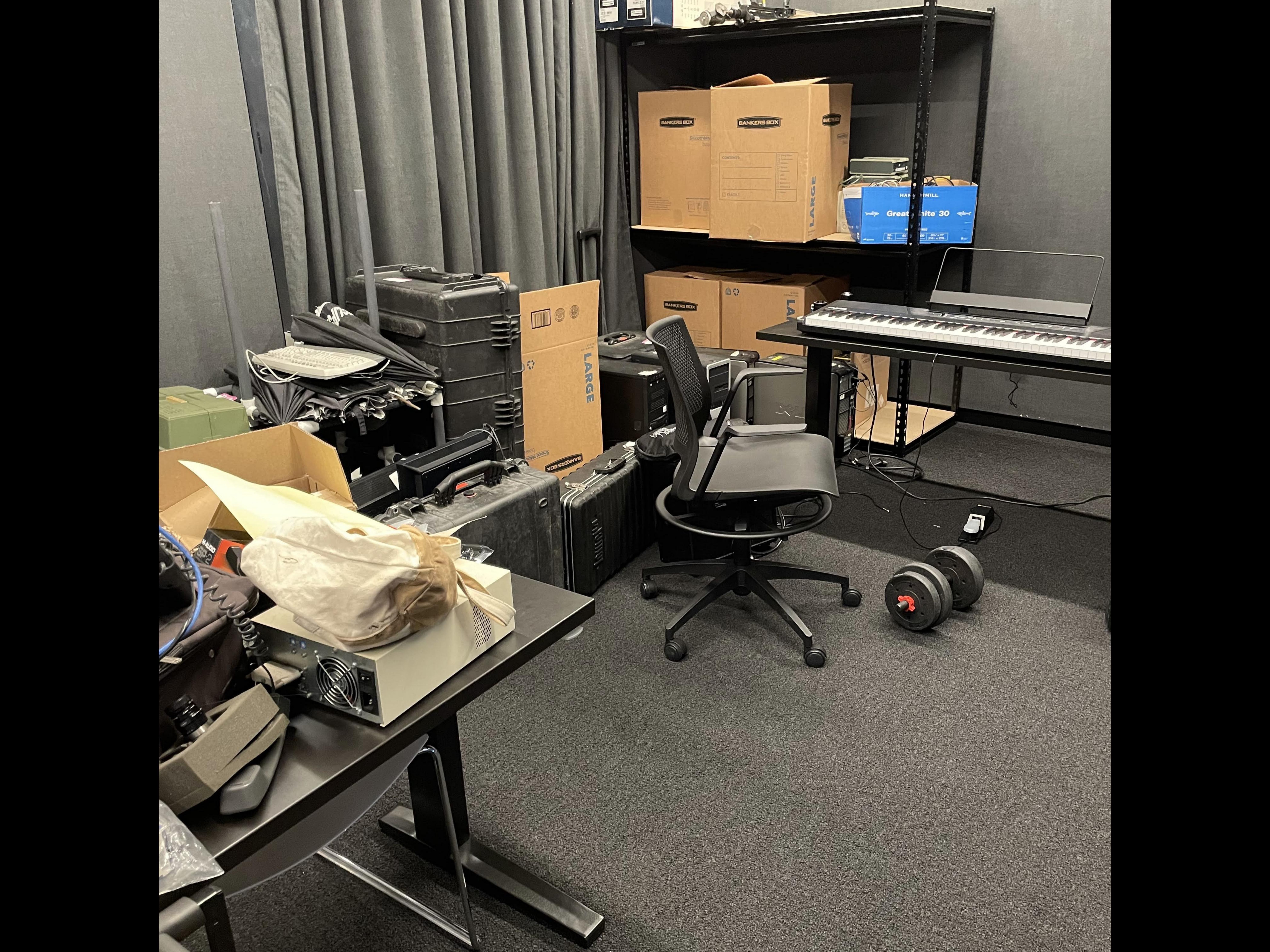}\\[-1pt]
        {\scriptsize (a) Frame 1}
    \end{minipage}\hfill
    \begin{minipage}[t]{0.30\textwidth}\centering
        \includegraphics[width=\textwidth]{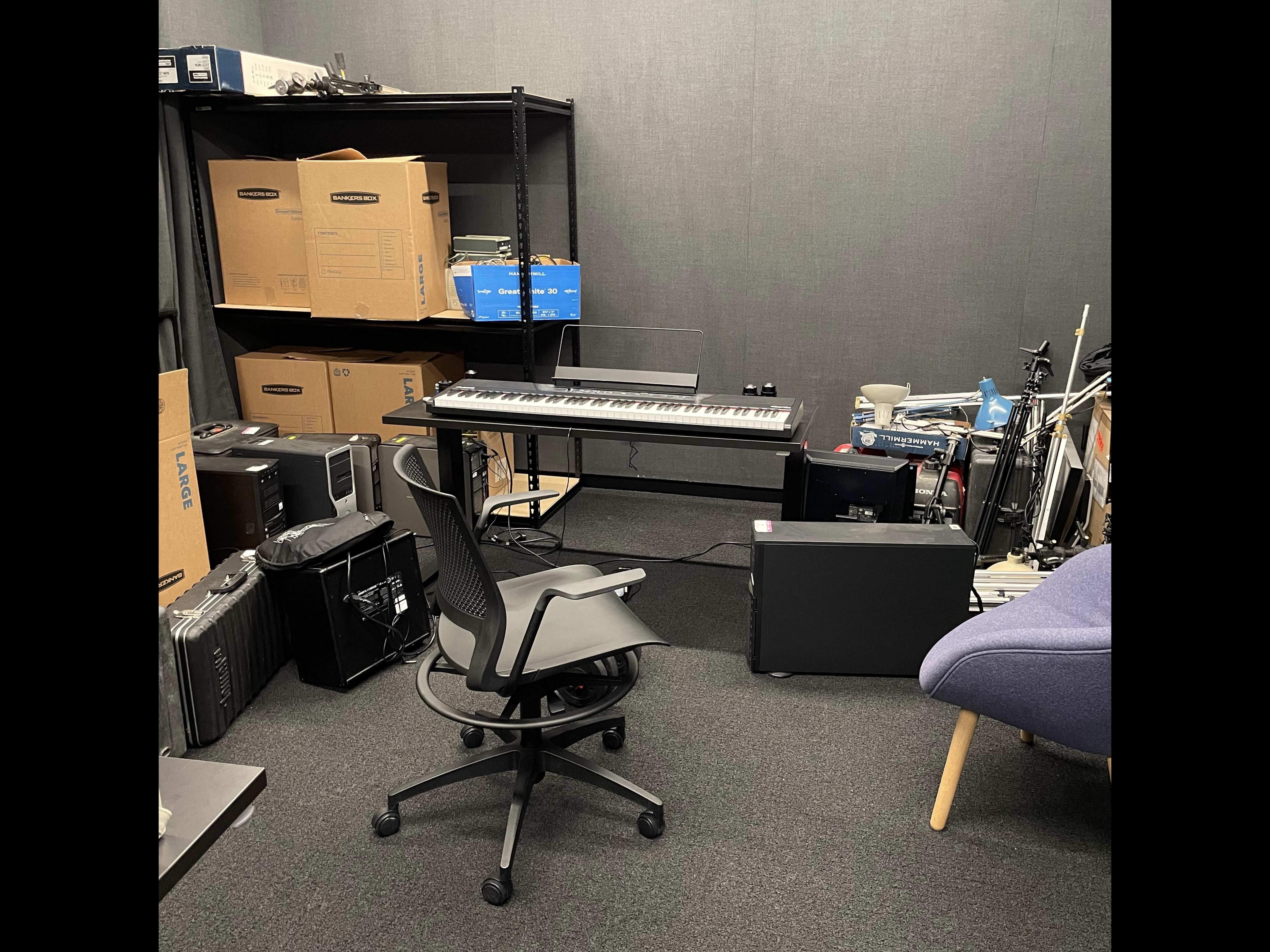}\\[-1pt]
        {\scriptsize (b) Frame 2}
    \end{minipage}\hfill
    \begin{minipage}[t]{0.30\textwidth}\centering
        \includegraphics[width=\textwidth]{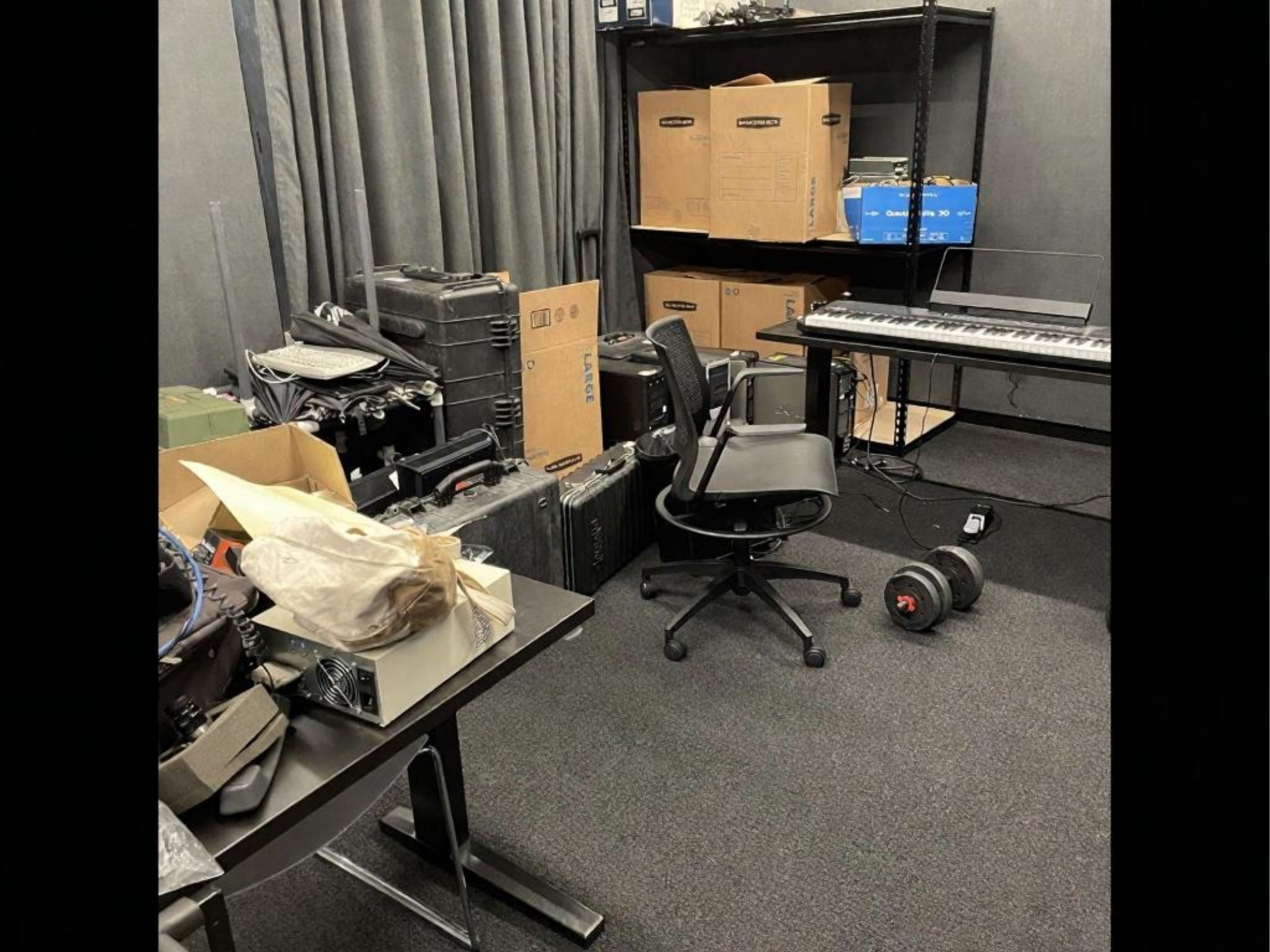}\\[-1pt]
        {\scriptsize (c) ThinkMorph VT}
    \end{minipage}
\end{minipage}\hfill
\begin{minipage}[t]{0.38\textwidth}
    \vspace{0pt}
    {\small\textbf{Question:} Were any objects moved from their original positions between frames?}
    \vspace{4pt}

    {\small\textbf{Options:}\\
    (A)~chair moved right \& towards camera\\
    (B)~chair moved left \& away from camera}
    \vspace{2pt}

    {\small\textbf{Ground truth:} A.}
\end{minipage}

\vspace{6pt}
\hrule
\vspace{4pt}

\begin{minipage}[t]{0.48\textwidth}
    \colorbox{red!6}{\parbox{\dimexpr\textwidth-2\fboxsep}{%
    {\small\textbf{ThinkMorph} \textcolor{red}{\ding{55}} $\to$ \textbf{B}}\\[2pt]
    \scriptsize
    \textbf{Round\,0:} Self-contradictory reasoning; first correctly observes ``moved left and slightly towards the camera,'' then states ``the answer is (A)\ldots \textit{However}\ldots the correct answer is (B).''\\[2pt]
    \textbf{VT generation:} produces image (c) with spatial distortions in chair position.\\[2pt]
    \textbf{Round\,1} (conditioned on VT): ``The chair is now more centrally placed\ldots its legs further from the camera\ldots the answer is (B).'' VT resolves internal ambiguity toward wrong direction.\\[2pt]
    $\Rightarrow$ \textbf{VT destabilizes uncertain reasoning.}
    }}
\end{minipage}\hfill
\begin{minipage}[t]{0.48\textwidth}
    \colorbox{green!6}{\parbox{\dimexpr\textwidth-2\fboxsep}{%
    {\small\textbf{\method} \textcolor{green!50!black}{\ding{51}} $\to$ \textbf{A}}\\[2pt]
    \scriptsize
    ``The chair, which was central in the first image, is now positioned slightly to the left and closer to the camera. Its backrest and seat are more directly facing the viewer, and the wheels are aligned with the carpet in a way that suggests it has been moved.''\\[4pt]
    $\Rightarrow$ \textbf{Consistent reasoning without self-contradiction.} No VT annotation to second-guess the correct initial spatial assessment.
    }}
\end{minipage}
\caption{\textbf{Case 4: VT-induced reasoning self-contradiction.} ThinkMorph's Round\,0 reasoning explicitly contradicts itself, and the VT image~(c) resolves this ambiguity toward the wrong answer. Without VT, \method reasons consistently and correctly. This illustrates how VT conditioning can \emph{actively degrade} reasoning quality.}
\label{fig:case4}
\end{figure*}

\vspace{0.5em}
\noindent\textbf{Summary.} The four cases above illustrate complementary VT interference mechanisms: (1)~self-reinforcing confirmation of initial errors (Case~1), (2)~pixel-level noise masking spatial discontinuities (Case~2), (3)~annotation-induced context stripping (Case~3), and (4)~reasoning destabilization through ambiguous VT cues (Case~4). In all cases, the \method student avoids these failure modes by reasoning directly from the original input, having internalized the generation pathway's spatial reasoning knowledge at the distribution level without inheriting its pixel-level limitations.

\clearpage
\section{Failure Cases: When High-Quality VT Help}
\label{app:failure-cases}

While \method outperforms ThinkMorph on complex spatial reasoning tasks where VT quality is low (Section~\ref{sec:analysis}), ThinkMorph retains an advantage on a subset of simpler benchmarks where its diffusion model produces high-quality VT images. Below we present representative failure cases from V* and MMVP. In these examples, the generated VT provide genuinely informative visual cues, e.g., magnified object regions or highlighted spatial details, that aid ThinkMorph's reasoning, whereas \method, lacking access to explicit VT images, misses these fine-grained cues despite having internalized the generation pathway's distributional knowledge.

\begin{figure*}[!ht]
\centering
\begin{minipage}[t]{0.60\textwidth}
    \vspace{0pt}
    \begin{minipage}[t]{0.45\textwidth}\centering
        \includegraphics[width=\textwidth]{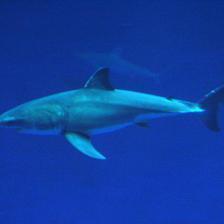}\\[-1pt]
        {\scriptsize (a) Input}
    \end{minipage}\hfill
    \begin{minipage}[t]{0.45\textwidth}\centering
        \includegraphics[width=\textwidth]{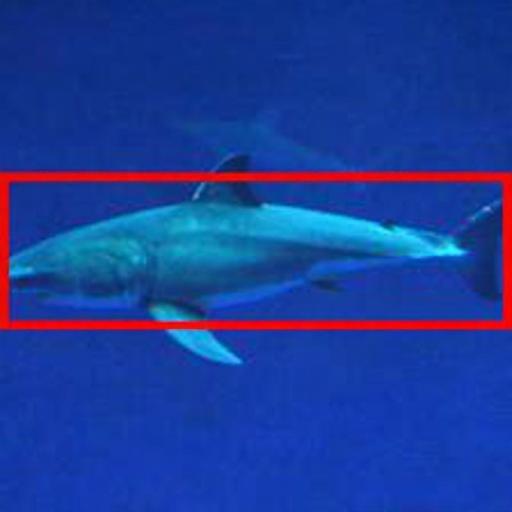}\\[-1pt]
        {\scriptsize (b) ThinkMorph VT}
    \end{minipage}\hfill
    \begin{minipage}[t]{0.28\textwidth}\centering
        \hfill
    \end{minipage}
\end{minipage}\hfill
\begin{minipage}[t]{0.38\textwidth}
    \vspace{0pt}
    {\small\textbf{Question:} Is the shark's belly visible in this image?}
    \vspace{4pt}

    {\small\textbf{Options:}\\
    (A)~Yes\\
    (B)~No}
    \vspace{2pt}

    {\small\textbf{Ground truth:} B (No).}
\end{minipage}

\vspace{6pt}
\hrule
\vspace{4pt}

\begin{minipage}[t]{0.48\textwidth}
    \colorbox{green!6}{\parbox{\dimexpr\textwidth-2\fboxsep}{%
    {\small\textbf{ThinkMorph} \textcolor{green!50!black}{\ding{51}} $\to$ \textbf{B (No)}}\\[2pt]
    \scriptsize
    \textbf{Round\,0:} ``The lower part of the shark, including its belly, is cut off by the edge of the frame\ldots there is no visible tissue continuing along the belly.''\\[2pt]
    \textbf{VT generation:} produces image~(b) with a bounding box around the shark's main body, explicitly delineating the region where the belly terminates at the frame boundary.\\[2pt]
    \textbf{Round\,1} (conditioned on VT): ``The area beneath the shark within this box does not show any continuation of its body\ldots the viewer is unable to see the shark's belly.'' Confirms B.\\[2pt]
    $\Rightarrow$ \textbf{VT bounding box clarifies the frame-cropping boundary.}
    }}
\end{minipage}\hfill
\begin{minipage}[t]{0.48\textwidth}
    \colorbox{red!6}{\parbox{\dimexpr\textwidth-2\fboxsep}{%
    {\small\textbf{\method} \textcolor{red}{\ding{55}} $\to$ \textbf{A (Yes)}}\\[2pt]
    \scriptsize
    ``The area beneath it is clearly visible\ldots the lighting and angle highlight the belly. The shark's fins, tail, and upper body are visible, but the belly is also distinctly seen, with no major visual blockages.''\\[4pt]
    $\Rightarrow$ \textbf{Without the VT bounding box, \method conflates the visible underside with the cropped belly region.}
    }}
\end{minipage}
\caption{\textbf{Failure Case 1 (MMVP): VT bounding box clarifies spatial occlusion.} The task requires determining whether the shark's belly is visible despite the image being cropped at the bottom. ThinkMorph's high-quality VT~(b) draws a bounding box delineating the visible body region, making the cropping boundary explicit and leading to the correct answer. \method, lacking this visual annotation, mistakes the partially visible underside for the full belly.}
\label{fig:fail1}
\end{figure*}

\begin{figure*}[!ht]
\centering
\begin{minipage}[t]{0.60\textwidth}
    \vspace{0pt}
    \begin{minipage}[t]{0.45\textwidth}\centering
        \includegraphics[width=\textwidth]{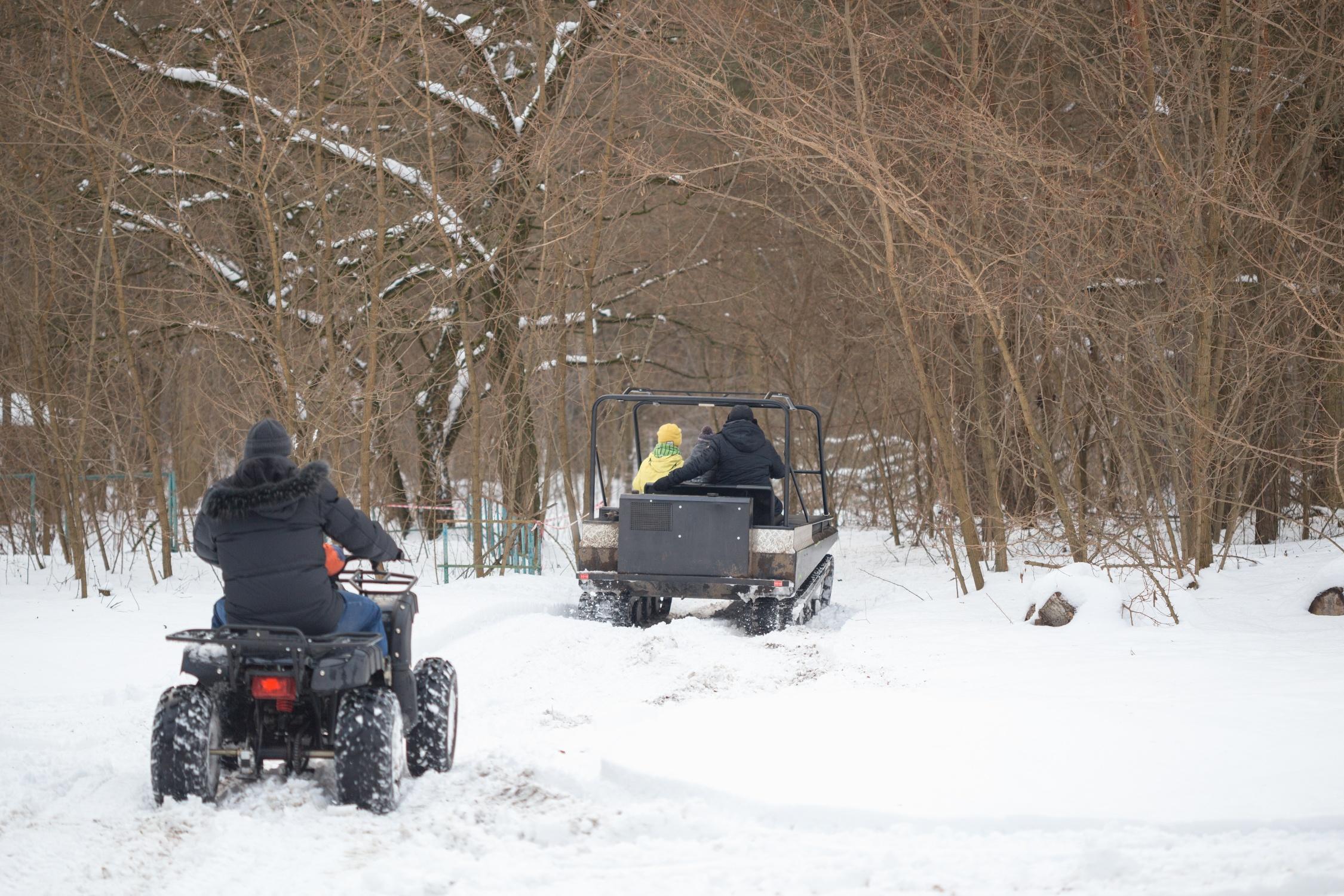}\\[-1pt]
        {\scriptsize (a) Input}
    \end{minipage}\hfill
    \begin{minipage}[t]{0.45\textwidth}\centering
        \includegraphics[width=\textwidth]{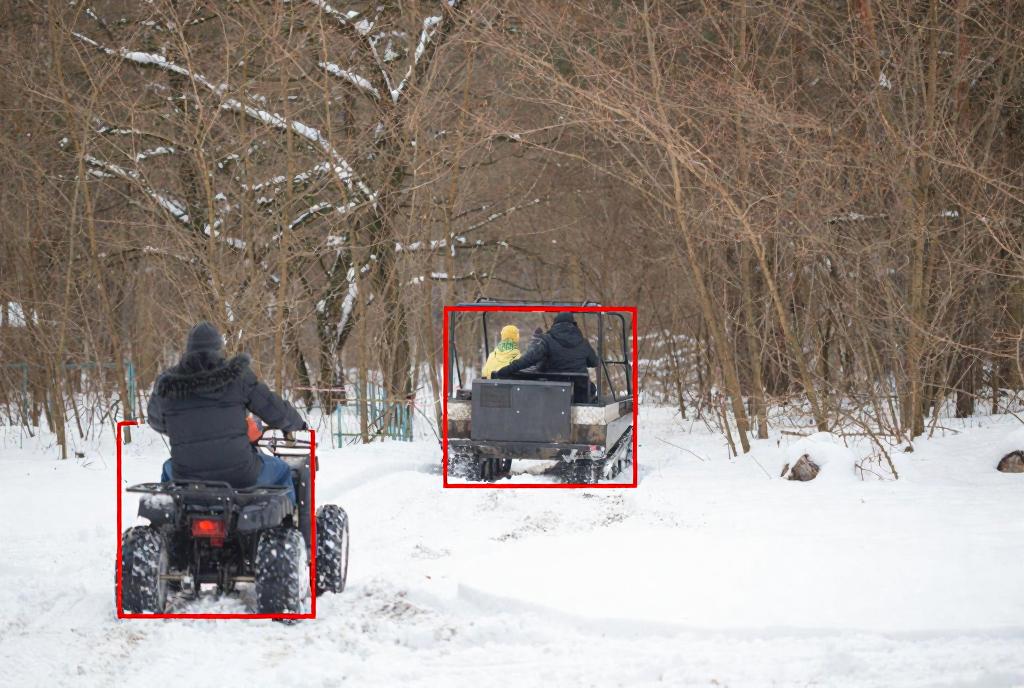}\\[-1pt]
        {\scriptsize (b) ThinkMorph VT}
    \end{minipage}\hfill
    \begin{minipage}[t]{0.28\textwidth}\centering
        \hfill
    \end{minipage}
\end{minipage}\hfill
\begin{minipage}[t]{0.38\textwidth}
    \vspace{0pt}
    {\small\textbf{Question:} What is the color of the scarf?}
    \vspace{4pt}

    {\small\textbf{Options:}\\
    (A)~red \quad (B)~green\\
    (C)~white \quad (D)~yellow}
    \vspace{2pt}

    {\small\textbf{Ground truth:} B (green).}
\end{minipage}

\vspace{6pt}
\hrule
\vspace{4pt}

\begin{minipage}[t]{0.48\textwidth}
    \colorbox{green!6}{\parbox{\dimexpr\textwidth-2\fboxsep}{%
    {\small\textbf{ThinkMorph} \textcolor{green!50!black}{\ding{51}} $\to$ \textbf{B (green)}}\\[2pt]
    \scriptsize
    \textbf{Round\,0:} ``This person is wearing a yellow jacket and a green scarf, which is visible around their neck.'' Correctly distinguishes scarf from jacket.\\[2pt]
    \textbf{VT generation:} produces image~(b) with bounding boxes around both the ATV rider and the lead vehicle passenger, isolating the person with the scarf.\\[2pt]
    \textbf{Round\,1} (conditioned on VT): ``The scarf stands out distinctly and is green in color, visible against the yellow jacket and the white snow.'' Confirms B.\\[2pt]
    $\Rightarrow$ \textbf{VT bounding boxes isolate the target person and scarf.}
    }}
\end{minipage}\hfill
\begin{minipage}[t]{0.48\textwidth}
    \colorbox{red!6}{\parbox{\dimexpr\textwidth-2\fboxsep}{%
    {\small\textbf{\method} \textcolor{red}{\ding{55}} $\to$ \textbf{D (yellow)}}\\[2pt]
    \scriptsize
    ``The person in the foreground, riding an ATV, is wearing a dark jacket and a scarf. The scarf is visible around the neck and is distinctly colored\ldots it is a bright yellow.''\\[4pt]
    $\Rightarrow$ \textbf{Without VT isolation, \method confuses the scarf's color with the adjacent yellow jacket.}
    }}
\end{minipage}
\caption{\textbf{Failure Case 2 (V*): VT bounding boxes disambiguate adjacent colors.} The scene contains a person in a yellow jacket with a green scarf, two colors in close proximity on a small, distant figure. ThinkMorph's VT~(b) draws bounding boxes around the relevant persons, enabling precise color discrimination between the jacket and scarf. \method, lacking this visual isolation, conflates the scarf's green with the jacket's yellow.}
\label{fig:fail2}
\end{figure*}

\vspace{0.5em}
\noindent\textbf{Discussion.} These failure cases reveal the boundary of distribution-level distillation. In both cases, the tasks are perceptually straightforward (requiring visibility judgment for Case~1 or fine-grained color discrimination for Case~2), and ThinkMorph's diffusion model produces high-quality VT annotations that genuinely aid reasoning. \method acquires the generation pathway's reasoning \emph{patterns} but cannot replicate the explicit visual operations (bounding-box delineation, region isolation) that VT provides at the pixel level. When VT quality is high, the benefit of filtering pixel-level noise is outweighed by the loss of these informative visual cues, resulting in the small net accuracy deficits observed on MMVP ($-$1.0pp) and V* ($-$2.1pp). Future work could explore hybrid strategies that selectively invoke VT generation when the expected quality is high and the task demands fine-grained perception.

\section{Evaluation Protocol}
\label{app:eval-protocol}

\paragraph{Internal models.} All internal models (BAGEL-7B, ThinkMorph, Text-only SFT, \methodnoise, \method) are evaluated using identical infrastructure: greedy decoding (temperature=0), max output tokens=1024, single H800 GPU, batch size 1. Images are preprocessed with NaViT-style patching (stride 14, max 980px, min 378px) identically across all systems. For ThinkMorph, VT generation uses 50 DDPM denoising steps with classifier-free guidance scale 3.5. Latency measurements include the full pipeline (image preprocessing, ViT encoding, LLM decoding, and VT generation where applicable) but exclude data loading I/O.

\paragraph{External VLMs.} InternVL3.5-8B/38B~\citep{wang2025internvl3} and Qwen3-VL-8B/32B~\citep{bai2025qwen3} are evaluated locally by the authors using official model weights and default inference configurations from their respective repositories. Janus-Pro-7B~\citep{chen2025janus} and Chameleon-7B~\citep{team2024chameleon} are similarly evaluated locally. All open models use greedy decoding with max 1024 output tokens and the same benchmark prompts as internal models. For V*~\citep{wu2024vstar}, we disable any external visual-search tools (e.g., the SEAL/V* visual search pipeline) for all evaluated VLMs, so every model is scored solely on its native single-pass multimodal reasoning capability, ensuring a fair comparison to our text-only \method student that has no tool access.

\paragraph{API models.} GPT-4o and GPT-5 results are obtained via the OpenAI API (May 2025 snapshot for GPT-4o, January 2026 for GPT-5) with temperature=0 and max\_tokens=1024. All API evaluations use the same prompt templates and image formats as internal evaluations.

\paragraph{Prompting.} Each benchmark uses a standardized prompt format: the input image(s) followed by the question text, with answer options listed when applicable. For multiple-choice benchmarks (VSP, BLINK, BLINK-J, MMVP, SAT, CV-Bench, V*), the prompt instructs the model to respond with the option letter. For open-ended benchmarks (ChartQA, VisPuzzle), the prompt requests a direct answer. 

\paragraph{\methodnoise details.} In \methodnoise, each real VT image in the teacher context is replaced with Gaussian noise $\mathcal{N}(0, 1)$ of the same spatial dimensions. The noise tensor is processed through the same ViT encoder as real VT images, producing tokens of the same shape and position encoding. All other elements of the teacher context (system prompt, problem image, question, reference intro, and transition prompt) are preserved unchanged. 

\end{document}